\documentclass{article}

\usepackage{microtype}
\usepackage{graphicx}
\usepackage{overpic}
\usepackage{booktabs} 
\usepackage{url,multirow,comment}

\usepackage{amsmath,amssymb,amsthm}
\usepackage{amsthm}
\usepackage{caption}

\usepackage{subcaption}
\usepackage{algorithm, algorithmic}
\usepackage{subcaption}

\usepackage[accepted]{icml2023}
\usepackage{authblk}

\newcommand{\pms}{\small{\pm}}
\newcommand{\til}[1]{\tilde{#1}}

\usepackage{hyperref}

\usepackage{amsmath}
\usepackage{amssymb}
\usepackage{mathtools}
\usepackage{amsthm}
\usepackage{enumitem}
\usepackage[capitalize,noabbrev]{cleveref}
\usepackage{booktabs}
\usepackage{float}
\usepackage{bbm}
\usepackage{thmtools,thm-restate}

\usepackage{rotating}

\theoremstyle{plain}
\newtheorem{theorem}{Theorem}
\newtheorem{lem}{Lemma}

\newtheorem{defn}{Definition}

\newtheorem{rem}{Remark}
\newtheorem{propty}{Property}

\newcommand{\E}[1]{\mathbb{E}\left[{#1}\right]}
\newcommand{\V}[1]{\mathrm{Var}\left[{#1}\right]}
\newcommand{\Esub}[2]{\mathbb{E}_{#1}\left[{#2}\right]}

\usepackage[textsize=tiny]{todonotes}

\usepackage{booktabs}
\usepackage{multirow}

\makeatletter
\def\thanks#1{\protected@xdef\@thanks{\@thanks
        \protect\footnotetext{#1}}}
\makeatother

\icmltitlerunning{Robust Counterfactual Explanations for Neural Networks With Probabilistic Guarantees}

\begin{document}

\twocolumn[
\icmltitle{Robust Counterfactual Explanations for Neural Networks\\ With Probabilistic Guarantees}



\begin{icmlauthorlist}
\icmlauthor{Faisal Hamman}{yyy}
\icmlauthor{Erfaun Noorani}{yyy}
\icmlauthor{Saumitra Mishra}{comp}
\icmlauthor{ Daniele Magazzeni}{comp}
\icmlauthor{Sanghamitra Dutta}{yyy}
\end{icmlauthorlist}

\icmlaffiliation{yyy}{Department of Electrical and Computer Engineering, University of Maryland, College Park.}
\icmlaffiliation{comp}{JP Morgan Chase AI Research}

\icmlcorrespondingauthor{Faisal Hamman}{fhamman@umd.edu}

\
\begin{center}
\fontsize{9.5pt}{0pt}\selectfont
\vspace{-0.5cm}

     \textsuperscript{1}University of Maryland, College Park \:
     \textsuperscript{2}JP Morgan AI Research
\vspace{-0.5cm}
\end{center}


\icmlkeywords{Explainability, Counterfactual Explanations, Machine Learning}

\vskip 0.4in
]



\printAffiliationsAndNotice{}  

\begin{abstract}
There is an emerging interest in generating robust counterfactual explanations that would remain valid if the model is updated or changed even slightly. Towards finding robust counterfactuals, existing literature often assumes that the original model $m$ and the new model $M$ are bounded in the parameter space, i.e., $\|\text{Params}(M){-}\text{Params}(m)\|{<}\Delta$. However, models can often change significantly in the parameter space with little to no change in their predictions or accuracy on the given dataset. In this work, we introduce a mathematical abstraction termed \emph{naturally-occurring} model change, which allows for arbitrary changes in the parameter space such that the change in predictions on points that lie on the data manifold is limited. Next, we propose a measure -- that we call \emph{Stability} -- to quantify the robustness of counterfactuals to potential model changes for differentiable models, e.g., neural networks. Our main contribution is to show that counterfactuals with sufficiently high value of \emph{Stability} as defined by our measure will remain valid after potential ``naturally-occurring'' model changes with high probability (leveraging concentration bounds for Lipschitz function of independent Gaussians). Since our quantification depends on the local Lipschitz constant around a data point which is not always available, we also examine practical relaxations of our proposed measure and demonstrate experimentally how they can be incorporated to find robust counterfactuals for neural networks that are close, realistic, and remain valid after potential model changes \footnotemark . This work also has interesting connections with model multiplicity, also known as, the Rashomon effect.
\end{abstract}

\section{Introduction}
\label{introduction}
Counterfactual explanations~\cite{wachter2017counterfactual,Karimi_arXiv_2020,barocas2020hidden} have garnered significant interest in various high-stakes applications, such as lending, hiring, etc. Counterfactual explanations aim to guide an applicant on how they can change a model outcome by providing suggestions for improvement. Given an original data-point (e.g., an applicant who is denied a loan), the goal is to try to find a point on the other (desired) side of the decision boundary (a hypothetical applicant who is approved for the loan) which also satisfies several other preferred constraints, such as, (i) proximity to the original point; (ii) changes in as few features as possible; and (iii) conforming to the data manifold. Such a data-point that alters the model decision is widely referred to as a ``counterfactual.''

However, in several real-world scenarios, such as credit lending, the models making these high-stakes decisions have to be updated due to various reasons~\cite{upadhyay2021towards,black2021consistent}, e.g., to retrain on a few additional data points, change the hyper-parameters or seed, or transition to a different model class~\cite{pawelczyk2020learning}. Such model changes can often cause the counterfactuals to become invalid because typically they are quite close to the original data point, and hence, also quite close to the decision boundary. For instance, suppose the counterfactual explanation suggests an applicant to increase their income by $10K$ to get approved for a loan and they actually act upon that, but now, due to updates to the original model, they are still denied by the updated model.\footnotetext{Implementation is  available at \url{https://github.com/FaisalHamman/TReX-Counterfactuals}}

If counterfactuals become invalid due to model updates, this can lead to confusion and distrust in the use of algorithms in high-stakes applications. Users would typically act on the suggested counterfactuals over some time, e.g., increase their income for credit lending, but only to find that it is no longer enough since the model has slightly changed (perhaps due to retraining with a new seed or hyperparameter). This cycle of invalidation and regenerating new counterfactuals can not only be frustrating and time-consuming for users but also potentially hurt an institution's reputation.

This motivates our primary question:
\emph{How do we provide theoretical guarantees on the robustness of counterfactuals to potential model changes?}

Towards addressing this question, in this work, we introduce the abstraction of ``naturally-occurring'' model change for differentiable models. Our abstraction allows for arbitrary changes in the parameter space such that the change in predictions on points that lie on the data manifold is limited. This abstraction motivates a measure of robustness for counterfactuals that arrives with provable  probabilistic guarantees on their validity under naturally-occurring model change. We also introduce the notion of ``targeted'' model change and provide an impossibility result for such model change. Next, by leveraging a \emph{computable relaxation} of our proposed measure of robustness, we then design and implement algorithms to find robust counterfactuals for neural networks. Our experimental results validate our theoretical understanding and illustrate the efficacy of our proposed algorithms. We summarize our contributions here: 

\begin{itemize} [leftmargin=*, topsep=0pt, itemsep=0pt]
\item \textbf{Abstraction of ``naturally-occurring'' model change for differentiable models:} 
Existing literature~\cite{upadhyay2021towards,black2021consistent} on robust counterfactuals often assumes that the original model $m$ and the new model $M$ are bounded in the parameter space, i.e., $\|\text{Params}(M)-\text{Params}(m)\|$$<$$\Delta$. Building on \citet{dutta2022robust} for tree-based models, we note that models can often change significantly in the parameter space with little to no change on their predictions or accuracy on the given dataset. To capture this, we introduce an abstraction (see Definition~\ref{defn:natural}), that we call \emph{naturally-occurring} model change, which instead allows for arbitrary changes in the parameter space such that the change in predictions on points that lie on the data manifold is limited. Our proposed abstraction of naturally-occurring model change also has interesting connections with predictive/model multiplicity, also known as, the Rashomon Effect~\citep{Breiman2001StatisticalMT,marx2020predictive}.

\item \textbf{A measure of robustness with probabilistic guarantees on validity:}
Next, we propose a novel mathematical measure -- that we call Stability -- to quantify the robustness of counterfactuals to potential model changes for differentiable models. Stability of a counterfactual $x \in \mathbb{R}^d$ with respect to a model $m(\cdot)$ is given by: $$R_{k,\sigma^2}(x,m)=\frac{1}{k}\sum_{x_i \in N_{x,k}} \left( m(x_i)  - \gamma_x \|x-x_i\|\right),$$ where $N_{x,k}$ is a set of $k$ points in $\mathbb{R}^d$ drawn from the Gaussian distribution $\mathcal{N}(x,\sigma^2\mathrm{I}_{d})$ with $\mathrm{I}_{d}$ being the identity matrix, and $\gamma_x$ is the local Lipschitz constant of the model $m(\cdot)$ around $x$ (Definition~\ref{prop:robustness}). 

Our main contribution in this work is to provide a theoretical guarantee (see Theorem~\ref{thm:guarantee}) that counterfactuals with a sufficiently high value of Stability (as defined by our measure) will remain valid with high probability after ``naturally-occurring'' model change. Our result leverages concentration bounds for Lipschitz functions of independent Gaussian random variables (see Lemma~\ref{gauus}). 

Since our proposed Stability measure depends on the local Lipschitz constant which is not always available, we also examine practical relaxations of our measure of the form: $$\hat{R}_{k,\sigma^2}(x,m)=\frac{1}{k}\sum_{x_i \in N_{x,k}} \left(m(x_i) - |m(x)-m(x_i)| \right).$$ The first term essentially captures the mean value of the model output in a region around it (higher mean is expected to be more robust and reliable). The second term captures the local variability of the model output in around it (lower variability is expected to be more reliable). This intuition is in alignment with the results in~\citet{dutta2022robust} for tree-based models.

\item \textbf{Impossibility under targeted model change:} We also make a clear distinction between our proposed naturally-occurring and targeted model change. Under targeted model change, we provide an impossibility result (see Theorem~\ref{thm:impossibility}) that given any counterfactual for a model, one can always design a new model that is quite similar to the original model and that renders that particular counterfactual invalid. However, in this work, our focus is on non-targeted model change such as retraining on a few additional data points, changing some hyperparameters or seed, etc., for which we have defined the abstraction of ``naturally-occurring'' model change (see Definition~\ref{defn:natural}). 

\item \textbf{Experimental results:}
We explore methods for incorporating our relaxed measure into generating robust counterfactuals for neural networks. We introduce T-Rex:I (Algorithm \ref{alg:ReX}), which finds robust counterfactuals that are close to the original data point. T-Rex:I can be integrated into any base technique for generating counterfactuals to improve robustness. We also propose T-Rex:NN (Algorithm \ref{alg:LOF}), which generates robust counterfactuals that are data-supported, making them more realistic (along the lines of \citet{dutta2022robust} for tree-based models). Our experiments show that T-Rex:I can improve robustness for neural networks without significantly increasing the cost, and T-Rex:NN consistently generates counterfactuals that are similar to the data manifold, as measured using the local outlier factor (LOF).

\end{itemize}
\subsection{Related Works}
Counterfactual explanations have 
seen growing interest in recent years~\cite{verma2020counterfactual,Karimi_arXiv_2020,wachter2017counterfactual}. Regarding their robustness to model changes, \citet{pawelczyk2020counterfactual,kanamori2020dace,poyiadzi2020face} argue that counterfactuals situated on the data manifold are more likely to be more robust than the closest counterfactuals. Later, \citet{dutta2022robust} demonstrate that generating counterfactuals on the data manifold may not be sufficient for robustness. While the importance of robustness in local explanation methods has been emphasized \cite{Hancox-Li_fat_2020}, the problem of specifically generating robust counterfactuals has been less explored, with the notable exceptions of some recent works \cite{upadhyay2021towards,rawal2020can,black2021consistent,dutta2022robust,jiang2022formalising}. In \citet{upadhyay2021towards}, the authors propose an algorithm called ROAR that uses min-max optimization to find the \emph{closest} counterfactuals that are also robust. In \citet{rawal2020can}, the focus is on analytical trade-offs between validity and cost. \citet{jiang2022formalising} introduces a method for identifying close and robust counterfactuals based on a framework that utilizes interval neural networks. \citet{black2021consistent} propose that local Lipschitzness can be leveraged to generate consistent counterfactuals and propose an algorithm called Stable Neighbor Search to generate consistent counterfactuals for neural networks. Our research builds on this perspective and further performs Gaussian sampling around the counterfactual, leading to a novel estimator for which we are also able to provide probabilistic guarantees going beyond the bounded model change assumption. Furthermore, examining all three performance metrics,
namely, cost, validity (robustness), and likeness to the data-manifold has received less attention with the
notable exception of \citet{dutta2022robust} but they focus only on tree-based models (non-differentiable). We also refer to \citet{Mishra_arXiv_2021} for a survey.

We note that \citet{laugel2019issues,alvarez2018robustness} propose an alternate perspective of robustness in explanations (called $L$-stability in \citet{alvarez2018robustness}) which is built on similar individuals receiving similar explanations. \citet{pawelczyk2022probabilistically,maragno2023finding,dominguez2022adversarial} focus on  finding counterfactuals that are robust to small input perturbations (noisy counterfactuals).  In contrast, our focus is on counterfactuals remaining valid after some changes to the model, and providing theoretical guarantees thereof. 

Our work also shares interesting conceptual connections with a body of work on model multiplicity or predictive multiplicity, also known as the Rashomon effect~\citep{Breiman2001StatisticalMT,marx2020predictive,modelmult_black,hsu2022rashomon}. \citet{Breiman2001StatisticalMT} suggested that models can be very different from each other but have almost similar performance on the data manifold. The term predictive multiplicity was suggested by \citet{marx2020predictive}, defining it as the ability of a prediction problem to admit competing models with conflicting predictions and introduce formal measures to evaluate the severity of predictive multiplicity. \citet{modelmult_black} investigates ways to leverage model multiplicity beneficially in model selection processes while simultaneously addressing its concerning implications. \citet{watson2023predictive} offered a framework for measuring predictive multiplicity in classification, introducing measures that encapsulate the variation in risk estimates over the ensemble of competing models. \citet{hsu2022rashomon} unveiled a novel metric, Rashomon Capacity, for measuring predictive multiplicity in probabilistic classification. Our proposed abstraction of naturally-occurring model change, as explored in this work, can be viewed as a fresh perspective on model multiplicity that further emphasizes on the models that are more likely to occur.


\section{Preliminaries} 
Let $\mathcal{X} \subseteq \mathbb{R}^d$ denote the input space and let $\mathcal{S}{=}\{x_i \in \mathcal{X} \}_{i=1}^n $ be a dataset consisting of $n$ independent and identically distributed data points generated from a density $q$ over $\mathcal{X}$. We also let $m (\cdot):\mathbb{R}^d \to [0,1]$ denote the original machine learning model that takes a $d$-dimensional input value and produces an output probability lying between $0$ and $1$. The final decision is denoted by $\mathbbm{1}(m(x)\geq 0.5)$ where $\mathbbm{1}(\cdot)$ denotes the indicator function. 

\begin{defn}[$\gamma-$Lipschitz]\label{def:Lipschitz}
A function $m(\cdot)$ is said to be $\gamma-$Lipschitz if
$|m(x)-m(x')|{\leq} \gamma\|x-x'\|$ for all $ x,x' {\in} \mathbb{R}^d$.
\end{defn}
Here $\|\cdot\|$ denotes the Euclidean norm, i.e., for $u\in \mathbb{R}^d$, we have $\|u\|=\sqrt{u_1^2 +u_2^2+\ldots+u_d^2}$. In Remark~\ref{rem:relax}, we also discuss relaxations to local Lipschitz constants from global Lipschitz constants. We denote the updated or changed model as $M(\cdot):\mathbb{R}^d \to [0,1]$ where $M$ is a random entity. We mostly use capital letters to denote random entities, e.g., $M$, $X$, etc., and small letters to denote non-random entities, e.g., $m$, $x$, $\gamma$, $n$, etc.

\subsection{Background on Counterfactuals}

\begin{defn}[Closest Counterfactual $\mathcal{C}_{p}(x,m)$]\label{def:ClosCount}
Given $x\in \mathbb{R}^d$ such that $m(x)<0.5$, its closest counterfactual (in terms of $l_p$-norm) with respect to the model $m(\cdot)$ is defined as a point $x'\in \mathbb{R}^d$ that minimizes the $l_p$ norm $\|x-x'\|_p$ such that $m(x')\geq0.5$. 
\begin{equation}
\mathcal{C}_{p}(x,m)=\arg \min_{x'\in \mathbb{R}^d} \|x-x'\|_p 
\text{ such that } m(x')\geq0.5. \nonumber
\end{equation}  
\end{defn}
When one is interested in finding counterfactuals by changing as few features as possible, the $l_1$ norm is used (enforcing a sparsity constraint). These counterfactuals are also called \emph{sparse} counterfactuals~\cite{pawelczyk2020counterfactual}.

Closest counterfactuals often fall too far from the data manifold, resulting in unrealistic and anomalous instances, as noted in ~\citet{poyiadzi2020face,pawelczyk2020counterfactual,kanamori2020dace,verma2020counterfactual,Karimi_arXiv_2020,albini2021counterfactual}. This highlights the need for generating counterfactuals that lie on the data manifold.

\begin{defn}[Closest Data-Manifold Counterfactual $\mathcal{C}_{p,\mathcal{X}}(x,m)$]
\label{defn:data-support-CF}
Given $x\in \mathbb{R}^d$ such that $m(x)<0.5$, its closest data-manifold counterfactual $\mathcal{C}_{p,\mathcal{X}}(x,m)$ with respect to the model $m(\cdot)$ and data manifold $\mathcal{X}$ is defined as a point $x'\in \mathcal{X}$ that minimizes the $l_p$ norm $\|x-x'\|_p$ such that $m(x')\geq0.5$.
\begin{equation}
\mathcal{C}_{p,\mathcal{X}}(x,m)=\arg \min_{x'\in \mathcal{X}} \|x-x'\|_p
\text{ such that }{m(x')\geq0.5}. \nonumber
\end{equation}  
\end{defn}
In order to assess the similarity or anomalous nature of a point concerning the given dataset $\mathcal{S} \subseteq \mathcal{X}$, various metrics can be employed, e.g., K-nearest neighbors, Mahalanobis distance, Kernel density, LOF. These metrics play a crucial role in understanding the quality of counterfactual explanations generated by a model. One widely used metric in the literature on counterfactual explanations~\cite{pawelczyk2020counterfactual,kanamori2020dace,dutta2022robust} is the Local Outlier Factor (LOF).

\begin{defn}[Local Outlier Factor \cite{breunig2000lof}]
For $x \in \mathcal{S}$, let $L_k(x)$ be its $k$-Nearest Neighbors (k-NN) in $\mathcal{S}$. The $k$-reachability distance $rd_k$ of $x$ with respect to $x'$
is defined by $rd_k(x, x')= \max\{\delta(x, x'), d_k(x')\}$, where $d_k(x')$
is the distance $\delta$ between $x'$ and its $k$-th nearest instance
on $\mathcal{S}$. The $k$-local reachability density of $x$ is defined by
$lrd_k(x) = |L_k(x)| (
\sum_{x' \in L_k(x)} rd_k(x, x'))^{-1}.$ Then, the
k-LOF of $x$ on $\mathcal{S}$ is defined as follows:
$$LOF_{k,\mathcal{S}}(x) = \frac{1}{|L_k(x)|}
\sum_{x' \in L_k(x)}
\frac{lrd_k(x')}{lrd_k(x)}
.$$ Here, $\delta(x, x')$ is the distance between two $d$-dimensional feature vectors. The LOF Predicts $-1$ for anomalous points and $+1 $ for inlier points.
\label{defn:lof}
\end{defn}
\begin{figure}
\centering
\centerline{\includegraphics[width=0.6\columnwidth]{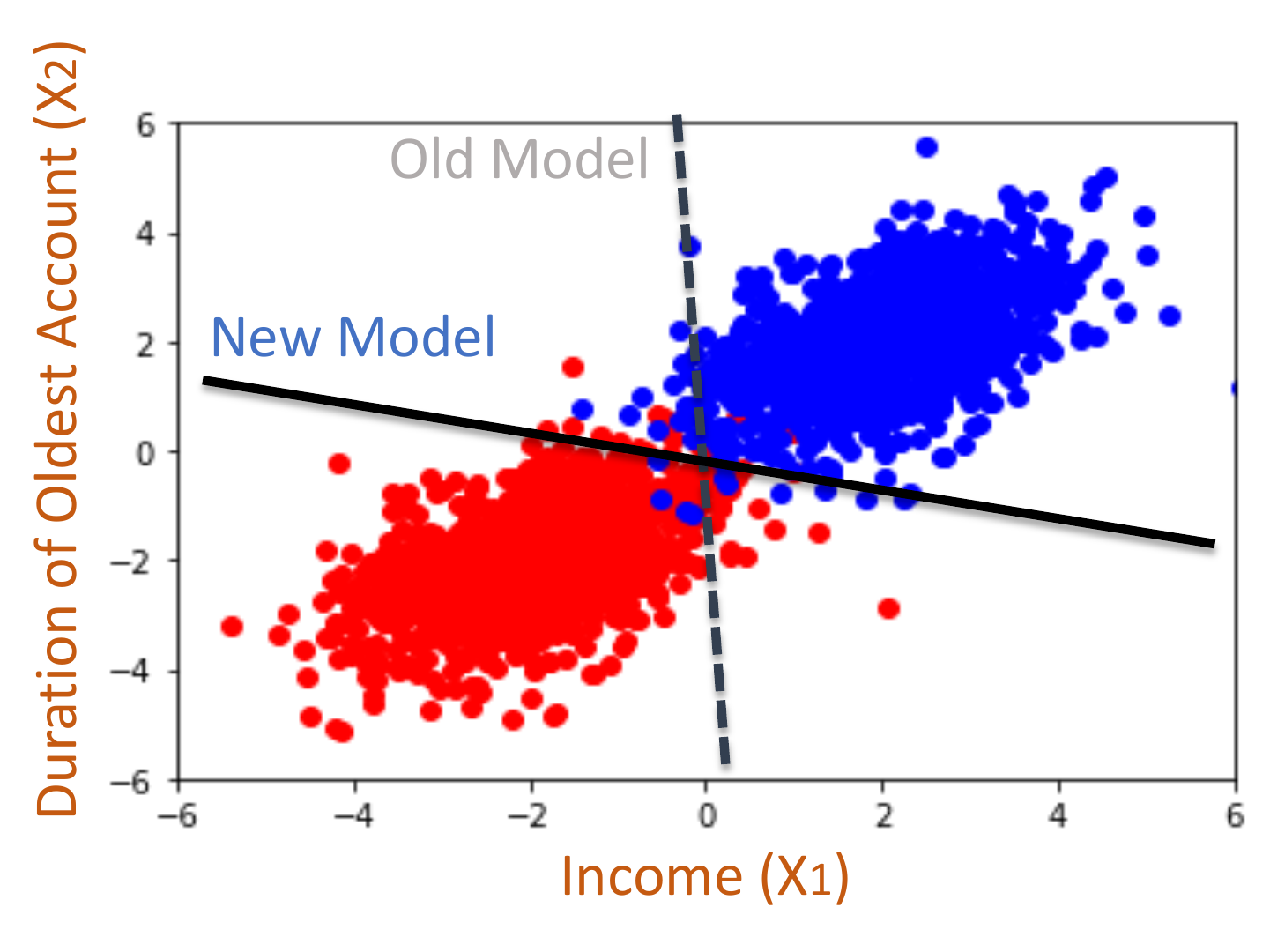}}
\caption{Models can often change drastically in the parameter space
causing little to no change in the actual decisions on the
points on the data manifold.}
\label{fig:motivation}
\end{figure}

\textbf{Goals:} In this work, our main goal is to provide \emph{probabilistic guarantees} on the robustness of counterfactuals to potential model changes for differential models such as neural networks. Towards achieving this goal, our objective involves: (i) introducing an abstraction that rigorously defines the class of model changes that we are interested in; and (ii) establishing a measure, denoted as $R_{\Phi}(x,m)$, for a counterfactual $x$ and a given model $m(\cdot)$, that quantifies its robustness to potential model changes. Here, $\Phi$ represents the hyperparameters of the robustness measure. Ideally, we desire that the measure $R_{\Phi}(x,m)$ should be high if the counterfactual $x$ is less likely to be invalidated by potential model changes. We seek to provide: (i) theoretical guarantees on the validity of counterfactuals with sufficiently high value of $R_{\Phi}(x,m)$; and also (ii) incorporate our measure $R_{\Phi}(x,m)$ into an algorithmic framework for generating robust counterfactuals which also meet other requirements, such as, low cost or likeness to the data manifold.

\section{Main Theoretical Contributions}
\label{sec:main}

In this section, we first introduce our proposed abstraction of \emph{naturally-occurring} model change and then propose a novel measure  – that we call \emph{Stability} – to quantify the robustness of counterfactuals to potential model changes. We derive a theoretical guarantee that counterfactuals that have a sufficiently high value of \emph{Stability} will remain valid after potential \emph{naturally-occurring} model change with high probability. But since our quantification would depend on the local Lipschitz constant around a data point, which is not always known, we also examine a practical relaxation of our proposed measure and demonstrate its applicability.

\subsection{Naturally-Occurring Model Change}
\begin{figure}[t]
    \centering
    \begin{overpic}[width=0.9\columnwidth]{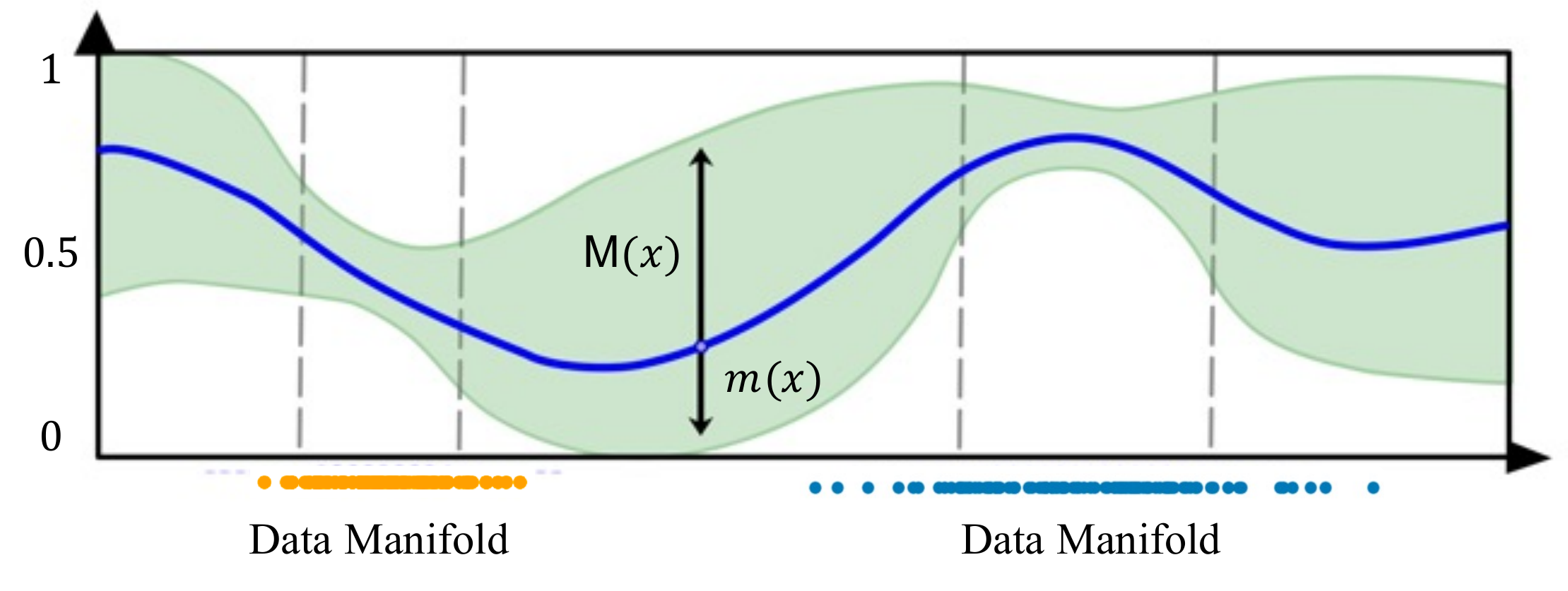}
\put (91,5) {$x$}
\end{overpic}
    \caption{Illustrates our proposed abstraction of naturally-occurring model change: The distribution of the changed model outputs $M(x)$ (stochastic) is centered around the original model output $m(x)$. The points specifically lying on the data-manifold act as anchors without much change as they exhibit lower variance in model outputs compared to points outside the manifold. This visualization also connects with the Rashomon effect, encapsulating the diverse yet similarly accurate models that can be learned from a given dataset. }
    \label{fig:manifoldMm}
\end{figure}

A popular assumption in existing literature~\cite{upadhyay2021towards,black2021consistent} to quantify potential model changes is to assume that the model changes are bounded in the parameter space, i.e., $$\|\text{Params}(M)-\text{Params}(m)\|<\Delta \text{ for a constant } \Delta.$$ Here, $\text{Params}(M)$ denote the parameters of the model $M$, e.g., weights of a neural network. However, we note that models can often change drastically in the parameter space causing little to no change on the actual decisions on the points on the data manifold (see Fig.~\ref{fig:motivation} for an example). In this work, we relax the bounded-model-change assumption, and instead introduce the notion of a naturally-occurring model change as defined in Definition~\ref{defn:natural}. Our abstraction allows for arbitrary model changes such that the change in predictions on points that lie on the data manifold is limited (see Fig. \ref{fig:manifoldMm} for an illustration).

\begin{defn}[Naturally-Occurring Model Change]\label{defn:natural} A naturally-occurring model change is defined as follows: 
\begin{enumerate}[leftmargin=*,topsep=0pt, itemsep=0pt]
\item $\E{M(X)|X=x}=\E{M(x)}=m(x)$ where the expectation is over the randomness of $M$ given a fixed value of $X=x\in \mathbb{R}^d$.\label{prop1}
\item Whenever $m(x)$ is $\gamma_{m}$-Lipschitz, any updated model $M(x)$ is also $\gamma-$Lipschitz for some constant $\gamma$. Note that, this constant $\gamma$ does not depend on $M$  since we may define $\gamma$ to be an upper bound on the Lipschitz constants for all possible $M$ as well as $m$. 
\item $\V{M(X)|X=x}=\V{M(x)}=\nu_x$ which depends on the fixed value of $X=x \in \mathbb{R}^d$. Furthermore, whenever $x$ lies on the data manifold $\mathcal{X}$, we have  $\nu_x \leq \nu$ for a small constant $\nu$.
\end{enumerate}
\end{defn}

Closely connected to naturally-occurring model change is the idea of the Rashomon effect, alternatively known as predictive or model multiplicity.  ~\cite{Breiman2001StatisticalMT,pawelczyk2020counterfactual,marx2020predictive, hsu2022rashomon} which suggests that models can be very different from each other but have almost similar performance on the data manifold, e.g., $\frac{1}{n}\sum_{i=1}^n|M(x_i)-m(x_i)|$ is small when the points $x_i$ lie on the data manifold. Under the naturally-occurring model change, this holds in expectation as follows:

\begin{restatable}[Connection to Roshomon Effect]{lem}{natocc}\label{lem:natural}
For points $x_1,\ldots,x_n \in \mathcal{X}$ (lying on the data-manifold) under naturally-occurring model change, the following holds:
\begin{equation}
\E{\frac{1}{n}\sum_{i=1}^n|M(x_i)-m(x_i)|} \leq \sqrt{\nu}.
\end{equation}
\end{restatable}
Thus, Definition~\ref{defn:natural} might be better suited over boundedness in the parameter space. Proof of Lemma~\ref{lem:natural} is in Appendix~\ref{apx:bound}.

\begin{rem}[Targeted Model Change] In contrast with naturally-occurring model change, we also introduce the notion of targeted model change (adversarial, worst-case) which essentially refers to a model change that is more deliberately targeted to make a particular counterfactual invalid. For example, one could have a new model $M(x)=m(x)$ almost everywhere except at or around the targeted point $x'$, i.e., $M(x')=1-m(x')$. See Section~\ref{subsec:Impossibility} for more details.
\end{rem}

\subsection{A Measure of Robustness With Probabilistic
Guarantees on Validity}

\subsubsection{Proposed Measure: Stability}

\begin{defn}[Stability] \label{def:stability}The stability of a counterfactual $x\in \mathbb{R}^d$ is defined as follows: 
\begin{align}
&R_{k,\sigma^2}(x,m)=\frac{1}{k}\sum_{x_i \in N_{x,k}} \left( m(x_i)  - \gamma \|x-x_i\|\right),
\end{align}
where $N_{x,k}$ is a set of $k$ points drawn from the Gaussian distribution $\mathcal{N}(x,\sigma^2\mathrm{I}_{d})$ with $\mathrm{I}_{d}$ being the identity matrix, and $\gamma$ is an upper bound on the Lipschitz constant for all models $M(\cdot)$ under naturally-occurring change. 
\label{prop:robustness}
\end{defn}

\begin{rem}[Relaxations to local Lipschitz]\label{rem:relax} While we prove our theoretical result (Theorem~\ref{thm:guarantee}) with the global Lipschitz constant $\gamma$, we can relax this to local Lipschitz constants $\gamma_x$, around a given point $x$. This is because we sample from a Gaussian centered around the point $x$ and hence mainly capture the variability around $x$. So most points will be very close to $x$ but a few points can still lie far away. Potential extensions of our guarantees could apply to truncated Gaussian and uniform sampling methods, given their sub-Gaussian properties. This is because Lipschitz concentration inherently extends to sub-Gaussian random variables ~\cite{baraniuk2008simple}.
\end{rem}



\subsubsection{Probabilistic Guarantee}

\begin{restatable}[Probabilistic Guarantee]{theorem}{Guarantee} Let $X_1,X_2,\ldots,X_k$ be $k$ iid random variables with distribution $\mathcal{N}(x,\sigma^2I_d)$ and $Z=\frac{1}{k}\sum_{i=1}^k (m(X_i)-M(X_i))$. Suppose $|\E{Z|M}- \E{Z}| < \epsilon'$.  Then, for any $ \epsilon > 2 \epsilon'$, a counterfactual $x \in \mathcal{X}$ under naturally-occurring model change satisfies:
\begin{equation*}
\Pr(M(x)\geq R_{k,\sigma^2}(x,m) {-} \epsilon) \!\! \geq 1 - \exp{\bigg(\frac{-k 
 {{\epsilon}}^2}{8(\gamma_m{+}\gamma)^2 \sigma^2}\bigg)}.
\end{equation*}
The probability is over the randomness of both $M$ and $X_i'$s.
\label{thm:guarantee}
\end{restatable}
\textbf{Intuition Behind Our Result:} 
This stability metric (Definition~\ref{prop:robustness}) is a way to measure the robustness of counterfactuals that are subject to natural model changes (see Definition~\ref{defn:natural}). The first term in the metric, represented by $\frac{1}{k}\sum_{i=1}^k m(X_i)$, captures the average model outputs for a group of points centered around the counterfactual $x$. The second term, represented by $\gamma \|x-X_i\|$, is an upper bound on the potential difference in outputs of any new model on the points $x$ and $X_i$ (Recall the Lipschitz property of $M$ around the point $x$).
Using our measure, the guarantee in Theorem \ref{thm:guarantee} can be rewritten as: 
\begin{multline}
\Pr\bigg(\frac{1}{k}\sum_{i=1}^k m(X_i) {-} M(x) {\leq}  \frac{\gamma}{k}\sum_{i=1}^k \|x{-}X_i\| {+} \epsilon\bigg) \\  \geq 1 - \exp{\left(\frac{-k\epsilon^2}{8(\gamma+\gamma_m)^2\sigma^2}\right)}.
\end{multline}
This form of the inequality allows for the following interpretation of Theorem~\ref{thm:guarantee}: The distance between the output of the new model on an input $x$, i.e., $M(x)$, and the average prediction of the neighborhood of the given input by the old model, i.e., $\frac{1}{k}\sum m(X_i)$ is upper bounded by $\epsilon$-corrected, $\gamma$ multiplied average distance of the datapoints within the neighborhood of the input $x$, i.e., $\frac{1}{k}\sum \|x-X_i\|$.

\textit{Proof Sketch:} The complete proof of Theorem~\ref{thm:guarantee} is provided in Appendix~\ref{app:prof_guar}. Here, we include a proof sketch. Notice that, using the Lipschitz property of $M(\cdot)$ around $x$, we have $M(x)  \geq M(X_i)-\gamma\|x-X_i\|$ for all $X_i$. Thus,
\begin{align}
M(x) & \geq \frac{1}{k}\sum_{i=1}^k (M(X_i)-\gamma\|x-X_i\|)  \\
& \overset{(a)}{\geq} \frac{1}{k}\sum_{i=1}^k (m(X_i)-\gamma\|x-X_i\|) -\epsilon,
\end{align}
where (a) holds from Lemma~\ref{lem:concentration_modified} with probability at least $1-\exp{\left(\frac{-k\epsilon^2}{8(\gamma+\gamma_m)^2\sigma^2}\right)}$.
\begin{restatable}[Deviation Bound]{lem}{bound}
\label{lem:concentration_modified}
Let $X_1,X_2,\ldots,X_k \sim \mathcal{N}(x,\sigma^2I_d)$ and $Z{=}\frac{1}{k}\sum_{i=1}^k (m(X_i){-}M(X_i))$. Suppose $|\E{Z|M}- \E{Z}| < \epsilon'.$ Then, under naturally-occurring model change, we have $\E{Z}{=}0$. Moreover, for any $\epsilon {>} 2 \epsilon',$
\begin{equation}
\Pr(Z \geq \epsilon) \leq \exp{\bigg(\frac{-k\epsilon^2}{8(\gamma + \gamma_m)^2\sigma^2}\bigg)}.
\end{equation}
\end{restatable}
\begin{proof}[Proof Sketch:]
The proof of Lemma~\ref{lem:concentration_modified} leverages concentration bounds for Lipschitz functions of independent Gaussian random variables (see Lemma~\ref{lem:concentration_actual}). The complete proof of Lemma~\ref{lem:concentration_modified} is provided in Appendix~\ref{apxproof}.
\end{proof}
\begin{restatable}[Gaussian Concentration Inequality]{lem}{gaus}\label{gauus}
Let $W=(W_1,W_2,\ldots,W_n)$ consist of $n$ i.i.d. random variables belonging to $\mathcal{N}(0,\sigma^2)$, and $Z=f(W)$ be a $\gamma$-Lipschitz function, i.e., $ |f(W)- f(W')|\leq \gamma \|W-W'\|.$
Then, we have,
\begin{equation}
\Pr(Z-\E{Z}\geq \epsilon) \leq \exp{\left(\frac{-\epsilon^2}{2\gamma^2\sigma^2}\right)} \text{ for all }\epsilon>0.
\end{equation}\label{lem:concentration_actual}
\end{restatable}
For the proof of Lemma~\ref{gauus} refer to \citet{boucheron2013concentration} in p.125. Our robustness guarantee (Theorem~\ref{thm:guarantee}) essentially states that $\Pr(M(x)\leq R_{k,\sigma^2}(x,m)-\epsilon) \leq \exp{\big(\frac{-k\epsilon^2}{8(\gamma+\gamma_m)^2\sigma^2}\big)}$ under naturally-occurring model change. For instance, if we find a counterfactual $x$ such that $R_{k,\sigma^2}(x,m)-\epsilon$ is greater or equal to $0.5$, then $M(x)$ would also be greater than $0.5$ with high probability. The term $\exp{\big(\frac{-k\epsilon^2}{8(\gamma+\gamma_m)^2\sigma^2}\big)}$ decays with $k$.

\subsubsection{Practical Relaxation of Stability and Its Properties}
\begin{figure*}
     \centering
    \begin{subfigure}[t]{0.16\textwidth}
        \raisebox{-\height}{\includegraphics[width=\textwidth]{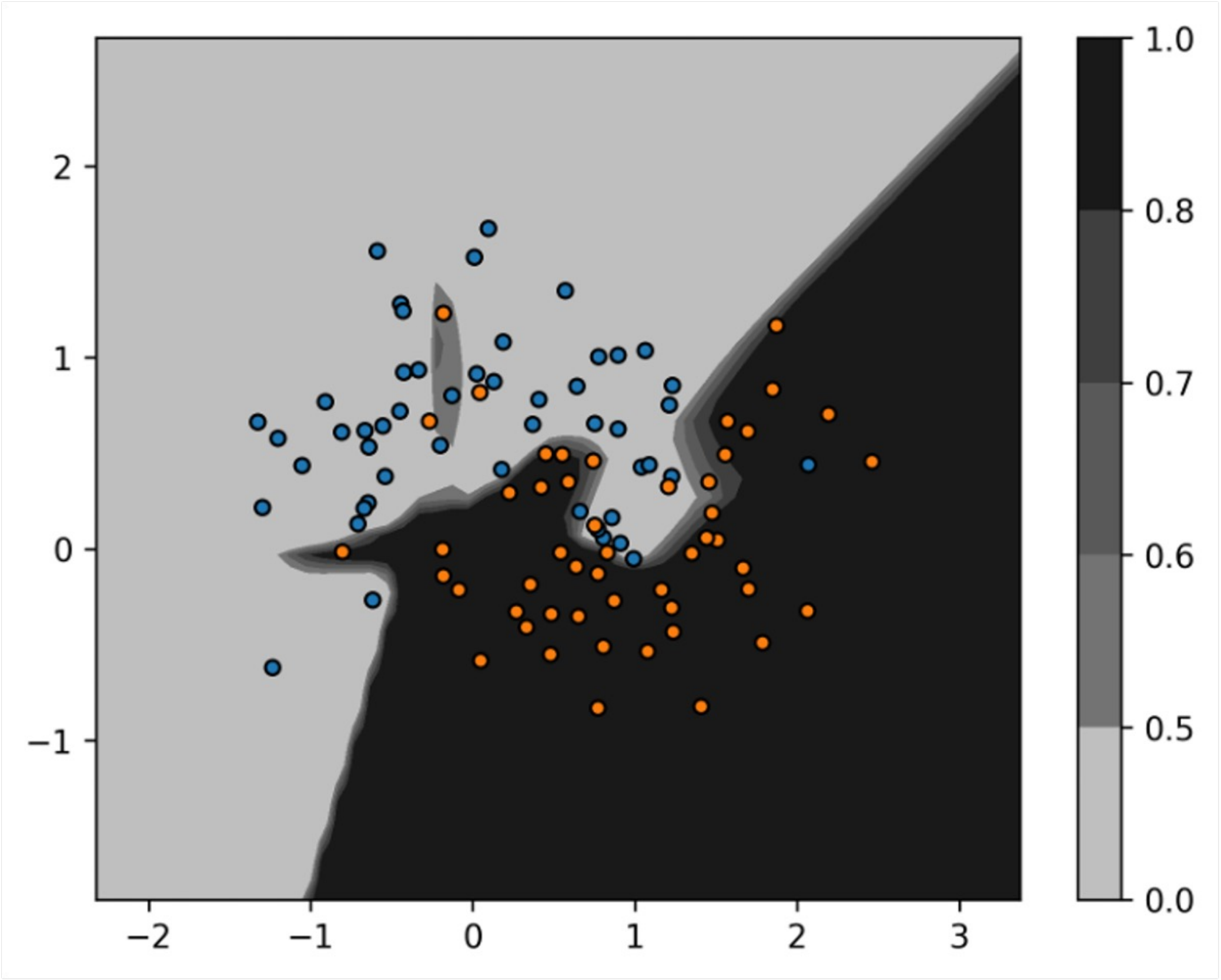}}
        \caption{$m$}
    \end{subfigure}
    \begin{subfigure}[t]{0.16\textwidth}
        \raisebox{-\height}{\includegraphics[width=\textwidth]{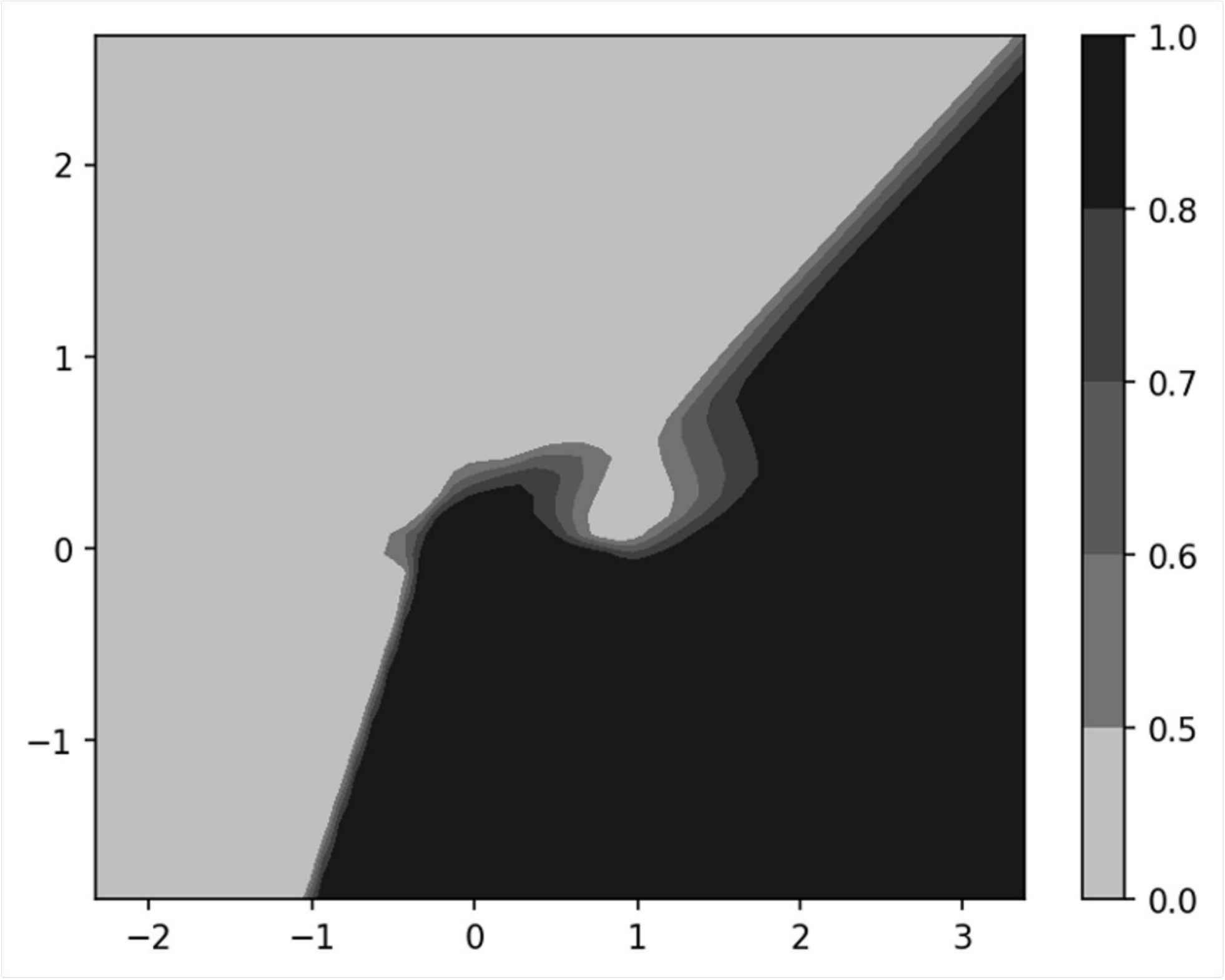}}
        \caption{$M=m_1$}
    \end{subfigure}
    \begin{subfigure}[t]{0.16\textwidth}
        \raisebox{-\height}{\includegraphics[width=\textwidth]{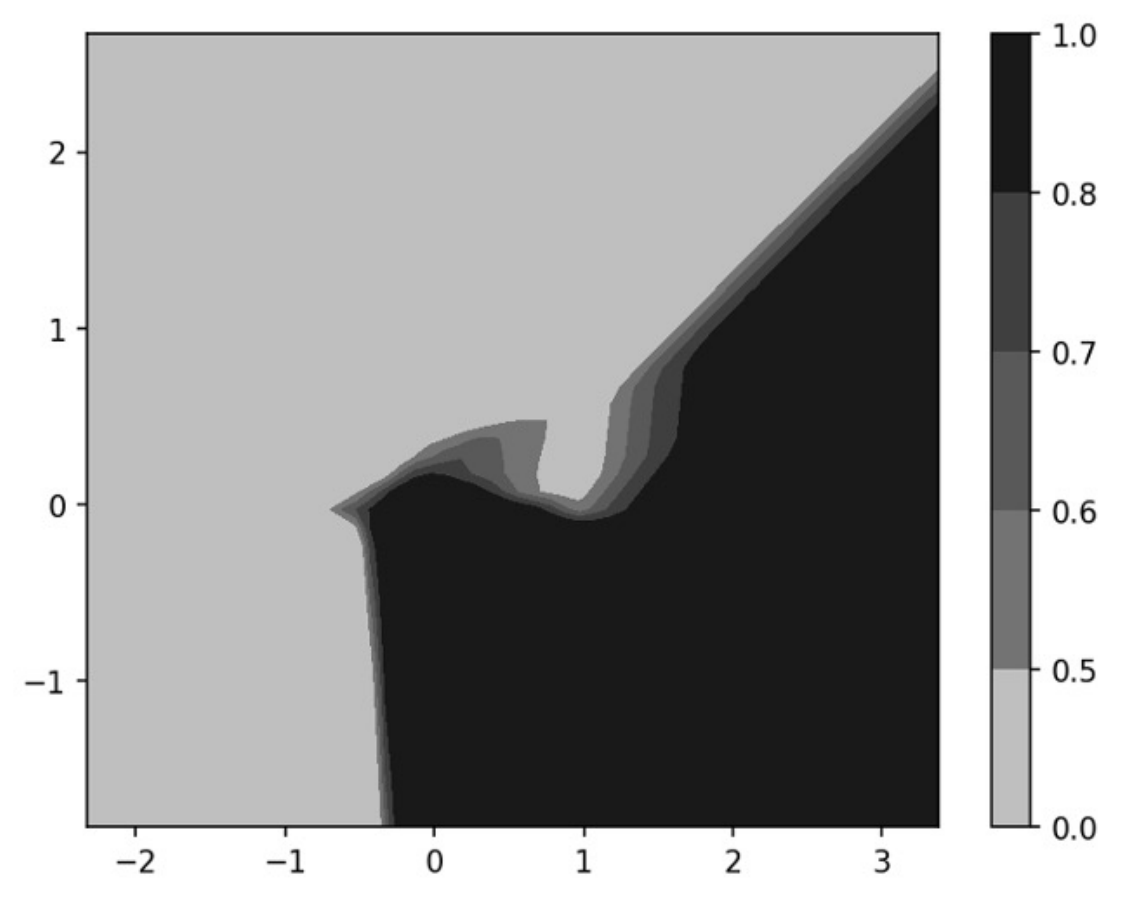}}
        \caption{$M=m_2$}
    \end{subfigure}
    \begin{subfigure}[t]{0.16\textwidth}
        \raisebox{-\height}{\includegraphics[width=\textwidth]{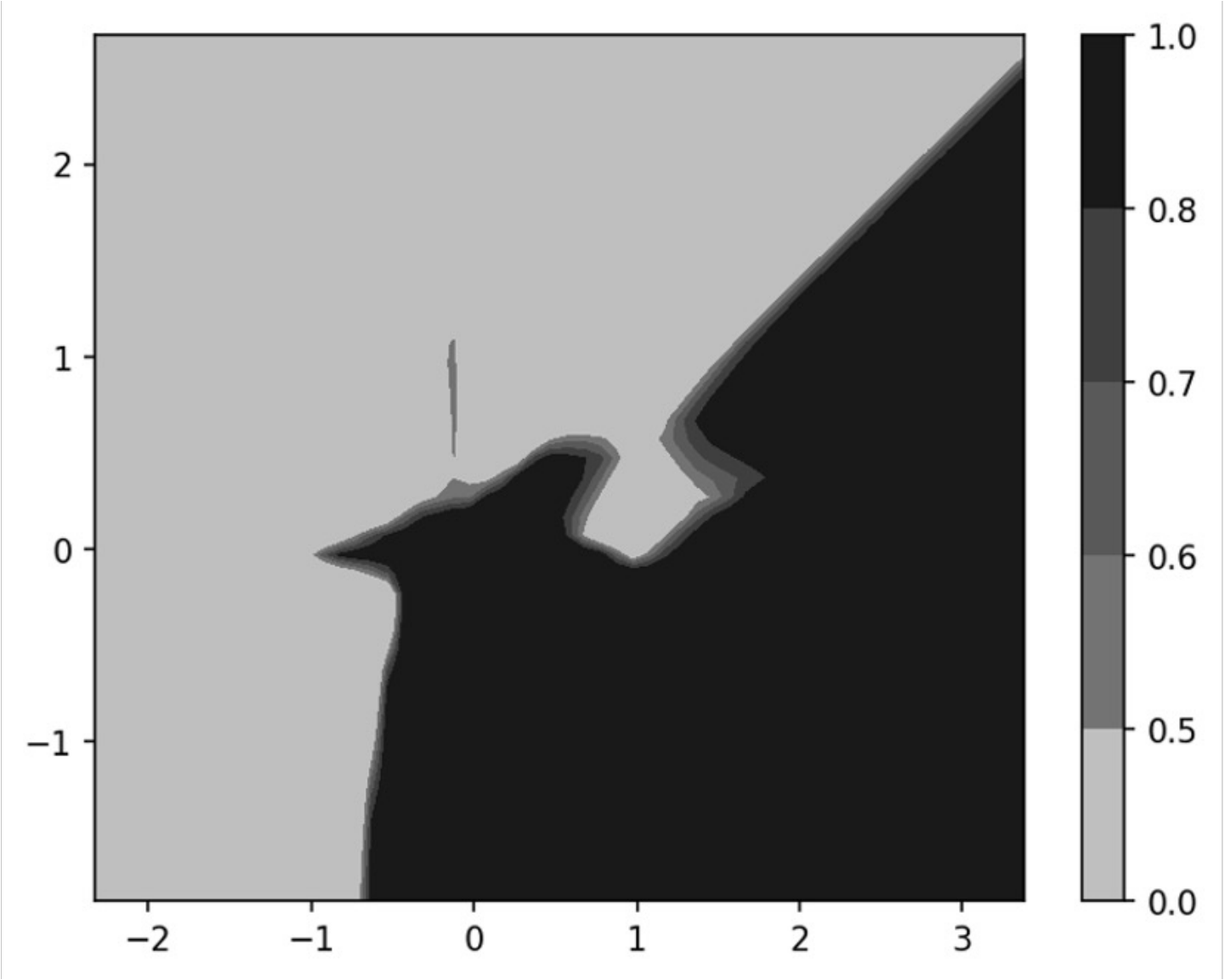}}
        \caption{$M=m_3$}
    \end{subfigure}
    \begin{subfigure}[t]{0.16\textwidth}
        \raisebox{-\height}{\includegraphics[width=\textwidth]{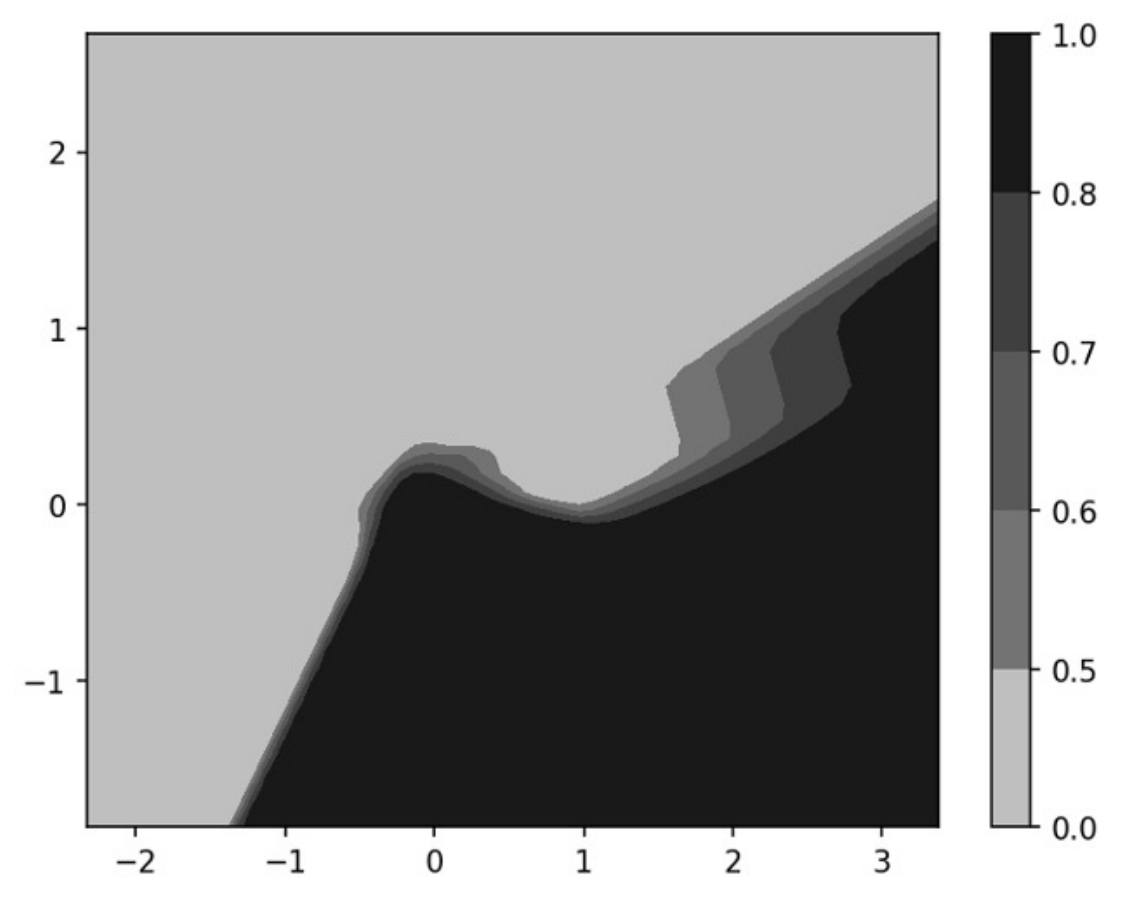}}
        \caption{$M=m_4$}
    \end{subfigure}
    \begin{subfigure}[t]{0.165\textwidth}
        \raisebox{-\height}{\includegraphics[width=\textwidth]{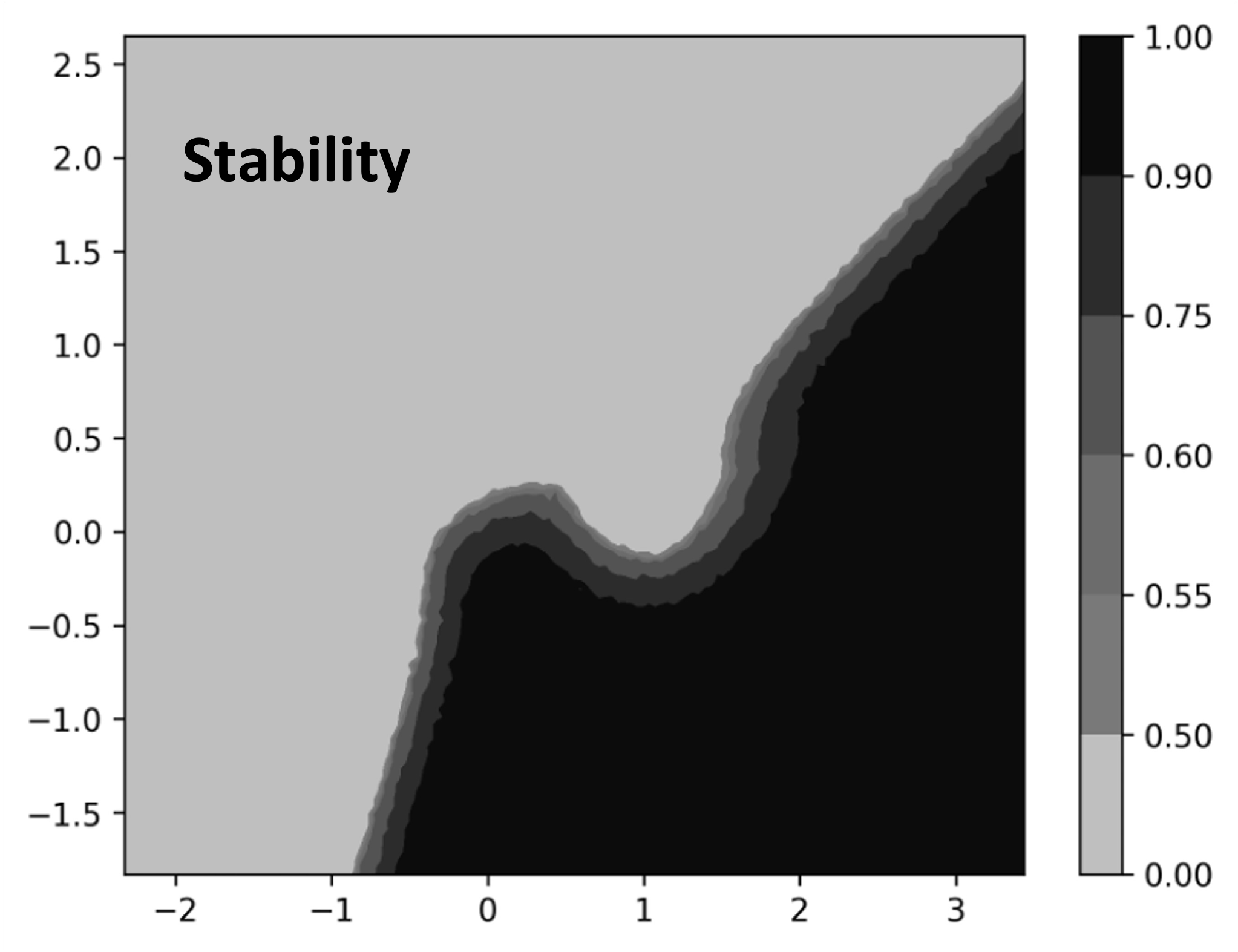}}
        \caption{{\small $\hat{R}_{k,\sigma^2}(x,m)$}}
    \end{subfigure}
     \caption{Effect of stability measure on naturally-occurring model changes: (a) corresponds to the original data distribution and the trained model. (b)-(e) demonstrate some examples of changed models obtained on retraining with different weight initializations. One may notice that the model decision boundary is changing a lot in the sparse regions of the data-manifold (few data-points), possibly violating the bounded-parameter change assumption but the predictions on the dense regions of the data-manifold do not change much (in alignment with Rashomon effect). This motivates our proposed abstraction of naturally-occurring model change which allows for arbitrary changes in the parameter space with little change in the actual predictions on the dense regions of the data manifold. (f)  demonstrates our proposed measure of stability $\hat{R}_{k,\sigma^2}(x,m)$ (high mean model output, low variability, \emph{almost} like a Gaussian filter) for which we derive probabilistic guarantees on validity. In essence, we show that under the abstraction of naturally-occurring model change, the stability measure captures the reliable intersecting region of changed models with high probability. In the original model, we observe that certain non-robust regions (i.e., those caused by overfitting to certain data points in the original model) have higher local Lipschitz values and variability. Counterfactuals assigned to these regions (even if $m(x)$ is high) would be invalidated in the changed models. The stability measure, which samples around a region, penalizes these higher local Lipschitz values. }
    \label{fig:matv}
\end{figure*}

While our proposed measure Stability (Definition~\ref{prop:robustness}) has probabilistic guarantees, we note that it relies on the Lipschitz constant $\gamma$ (or the local Lipschitz constant $\gamma_x$ around the point $x$), which is often unknown. Therefore, we next propose a practical relaxation of the measure as follows: 
\begin{defn}[Stability (Relaxed)] The stability (relaxed) of a counterfactual $x\in \mathbb{R}^d$ is defined as follows: 
\begin{align*}
\hat{R}_{k,\sigma^2}(x,m) =  \frac{1}{k}\sum_{x_i \in N_{x,k}}(m(x_i)  - \left|m(x) - m(x_i)\right|),
\end{align*}
where $N_{x,k}$ is a set of $k$ points  drawn from the Gaussian distribution $\mathcal{N}(x,\sigma^2\mathrm{I}_{d})$ with $\mathrm{I}_{d}$ being the identity matrix. 
\label{prop:robustness2}
\end{defn}
 To arrive at this relaxation, we utilize the Lipschitz property to approximate the aspect that involves the Lipschitz constant, specifically, by approximating $\gamma_x||x-x_i||$ with $|m(x)-m(x_i)|$. Another possibility is to consider an estimate of $\gamma_x$ given by:
\begin{equation}
  \label{eqn:relaxed_max}
\hat{\gamma}_x= \max_{x_i \in N_{x,k}} \frac{|m(x)-m(x_i)|}{\|x-x_i\|}.\end{equation}
We observed that the experimental results with both these stability estimates are in the same ballpark.

To gain a deeper understanding of the relaxed stability measure, we now consider some desirable properties of counterfactuals that make them more robust and then demonstrate that our proposed relaxation of Stability (Definition~\ref{prop:robustness2}) satisfies those desirable properties. These properties are inspired from~\citet{dutta2022robust} which proposed these properties for tree-based ensembles. The first property is based on the fact that the output of a model $m(x) \in [0,1]$ is expected to be higher if the model has more confidence in that prediction. 

\begin{propty}
For any $x \in \mathbb{R}^d$, a higher value of $m(x)$ makes it less likely to be invalidated due to model changes.
\label{propty:high_confidence}
\end{propty}
Having a high $m(x)$ alone does not guarantee robustness, as local variability around $x$ can make predictions less reliable. E.g., points with high $m(x)$ near the decision boundary are also vulnerable to invalidation with model changes. 

\begin{propty}
An $x {\in} \mathbb{R}^d$ is less likely to be invalidated if several points close to $x$ (denoted by $x'$) have a high value of $m(x')$.
\label{propty:high_confidence_mean}
\end{propty}
Counterfactuals may also be more likely to be invalidated if it lies in a highly variable region of the model output function. This is because the confidence of the model predictions in that region may be less reliable. 
\begin{propty}
An $x \in \mathbb{R}^d$ is less likely to be invalidated if model outputs around $x$ have low variability.
\label{propty:data_manifold}
\end{propty}

 Our stability measure aligns with these desired properties. Given a point $x \in \mathbb{R}^d$, it generates a set of $k$ points centered around $x$. The first term $\frac{1}{k} \sum_{x^{\prime} \in N_{x,k}} m\left(x^{\prime}\right)$ is expected to be high if the model output value $m(x)$ is high for $x$ as well as several points close to $x$. But the mean value of $m(x')$ around a point $x$ may not always capture the variability in that region, hence, the second term of our stability measure, i.e., $ \frac{1}{k}\sum_{x' \in N_{x,k}}\left|m(x) - m(x')\right|$. This term captures the variability of the model output values in a region around $x$. 
 
 It is worth noting that the variability term is only useful in conjunction with the mean term. This is because even points on the opposite side of the decision boundary can have varying levels of variability, regardless of whether $m(x')$ is less or greater than $0.5$.

In Fig. \ref{fig:matv}, we provide an example on the synthetic moon dataset to observe the effect of our stability measure on naturally-changed models. Note that these changed models were realized from actual experiments by retraining with different weight initializations.

\subsection{Impossibility Under Targeted Model Change} \label{subsec:Impossibility}
In this work, we make a key distinction between naturally-occurring and targeted model changes. While we are able to provide probabilistic guarantees for naturally-occurring model change, we also demonstrate an impossibility result for targeted model change. 

\begin{theorem}[Impossibility Under Targeted Change]
\label{thm:impossibility}
Given a model and a counterfactual, one can design another similar model such that the particular targeted counterfactual can be invalidated. 
\end{theorem}
What this result essentially demonstrates is that for a given model, one can design another similar model such that any particular targeted counterfactual can be invalidated. The proof relies on the possibility that one could have a new model $M(x)=m(x)$ almost everywhere except at or around the targeted point $x'$, i.e., $M(x')=1-m(x')$. 

\section{Generating Robust Counterfactuals using Our Proposed Measure: Stability} 
In this section, we examine two techniques of incorporating our proposed measure, Stability (relaxed; see Definition~\ref{prop:robustness2}), for generating robust counterfactuals for neural networks.

To begin, along the lines of \citet{dutta2022robust}, we first define a counterfactual robustness test.
\begin{defn}[Counterfactual Robustness Test] A counterfactual $x \in \mathbb{R}^d$ satisfies the robustness test if:
\begin{align}
    \hat{R}_{k,\sigma^2}(x,m)
    \geq \tau.
\end{align}
\end{defn}
Now, we would like to find a reasonable point $x'$ that receives a positive prediction from the model (essentially $m(x')\geq0.5$), while also satisfying the robustness test, $\hat{R}_{k,\sigma^2}(x',m) \geq \tau$. The threshold value of $\tau$ can be adjusted based on the desired effective validity (recall Theorem~\ref{thm:guarantee}). Hence, a larger threshold would likely ensure that the new model, $M$, remains valid with high probability. In trying to find a reasonable point $x'$, one may strive to generate robust counterfactuals that are as close as possible to the original point. One might also want the generated counterfactuals to be as realistic as possible, i.e., lie on the data manifold. Toward that end, we propose two algorithms.

We propose Algorithm~\ref{alg:ReX}, T-Rex:I, which incorporates our measure to find robust counterfactuals that are close to the original data point. T-Rex:I works with any preferred base method for generating counterfactuals. It evaluates the stability of the generated counterfactual and, if necessary, iteratively updates the generated counterfactual through a gradient ascent process until a robust counterfactual that meets the desired criteria is obtained.

\begin{rem}[Gradient of Stability]
In Algorithm \ref{alg:ReX}, we compute the gradient of $R(x,m)$ with respect to $x$ (not model parameters $m$). Such gradients w.r.t. $x$ instead of m are also computed commonly in adversarial machine learning and also in feature-attributions for explainability. We use TensorFlow \texttt{tf.GradientTape} for automatic differentiation, which allows for the computation of gradients with respect to certain inputs.
\end{rem}

To ensure that the counterfactuals are as realistic as possible, we also define the T-Rex:NN Counterfactual, which considers counterfactuals that lie within a dataset to avoid any unrealistic or anomalous results (see Algorithm~\ref{alg:LOF}).

\begin{defn}[Robust Nearest Neighbor Counterfactual]\label{defn:DSRC} Given $x \in \mathbb{R}^d$ such that $m(x)<0.5$, its robust nearest neighbor counterfactual $\mathcal{C}^{(\tau)}_{p, \mathcal{S}}(x, m)$ with respect to the model $m(\cdot)$ and dataset $\mathcal{S}$ is defined as another point $x^{\prime} \in \mathcal{S}$ that minimizes the $l_p$ norm $\left\|x-x^{\prime}\right\|_p$ such that $m\left(x^{\prime}\right)\geq0.5$ and $\hat{R}_{k,\sigma^2}(x',m) \geq \tau $.

\end{defn}
\begin{algorithm}[t]
\begin{algorithmic}\caption{T-Rex:I: Theoretically Robust EXplanations: Iterative Version}\label{alg:ReX}

   \STATE {\bfseries Input:} Model $m(\cdot)$, Datapoint $x$ with $m(x)<0.5$, \\Algorithm parameters ($k, \sigma^2, \eta, \tau,$ max$\_$steps).
   \STATE Generate initial counterfactual $x'$ using any technique. \\ Initialize robust counterfactual $x_c=x'$ and steps = 0.
   \WHILE{$\hat{R}_{k,\sigma^2}(x_c,m) < \tau$ and steps $<$ max$\_$steps }
   \STATE{Compute $\hat{R}_{k,\sigma^2}(x_c,m)$ \\Compute gradient $\nabla_{x_c}\hat{R}_{k,\sigma^2}(x_c,m)$}
   \STATE{Update $x_c$ via gradient ascent: \\\hspace{0.5cm} $x_c=x_c+\eta \nabla_{x_c}\hat{R}_{k,\sigma^2}(x_c,m)$}
   \STATE{Increment steps}
   \ENDWHILE
   \STATE{Output $x_c$ and exit}
\end{algorithmic}
\end{algorithm}

The closest data-supported counterfactual serves as a reliable reference, as it inherently has a high Local Outlier Factor (LOF). The T-Rex:I algorithm may find counterfactuals with lower costs, but they may compromise on the LOF and result in unrealistic samples. 

To address this, we propose Algorithm~\ref{alg:LOF}, T-Rex:NN, for finding data-supported counterfactuals. The algorithm first finds $k$ nearest neighbors counterfactuals of $x$ in dataset $\mathcal{S}$, checks through each of them to see if they satisfy the robustness test, $\hat{R}_{k,\sigma^2}(x',m) \geq \tau$, and terminates once such a counterfactual is found.

\section{Experiments}
\label{sec:experiments}

\begin{algorithm}[t]
\begin{algorithmic}\caption{T-Rex:NN: Theoretically Robust EXplanations: Nearest Neighbor Version}\label{alg:LOF}
   \STATE {\bfseries Input:} Model $m(\cdot)$, Datapoint $x$ with $m(x){<}0.5$, Dataset $\mathcal{S}$, Algorithm parameters ($K, \sigma^2, k, \tau$).
   \STATE{Find $K$ nearest neighbor counterfactuals $x_i' \in \mathcal{S}$ to $x$ with respect to model $m(\cdot)$, i.e., $\mathrm{NN}_x = (x_1',x_2',\ldots,x_K')$}.
   \FOR{ $x_i' \in \mathrm{NN}_x$}
   \STATE{Perform counterfactual robustness test on $x_i'$:\\ \hspace{0.5cm} Check if   $\hat{R}_{k,\sigma^2}(x_i',m) \geq \tau$}
   \IF{counterfactual robustness test is satisfied:}
   \STATE{Output $x_i'$ and exit}
    \ENDIF
    \ENDFOR
    \STATE{Output: \textit{No robust counterfactual found} and exit }
\end{algorithmic}
\end{algorithm}

In this section, we present experimental results to demonstrate the effectiveness of our proposed measure in capturing robustness and then generating robust counterfactuals that remain valid after potential model changes. We illustrate how our proposed Algorithm~\ref{alg:ReX} \& \ref{alg:LOF} utilizes our stability measure to effectively generate robust counterfactuals \footnote{Implementation is available at \url{https://github.com/FaisalHamman/TReX-Counterfactuals}}.

\textbf{Datasets:} We conduct experiments on several benchmark datasets, namely, HELOC~\cite{fico2018a}, German Credit, Cardiotocography (CTG), Adult~\cite{Dua:2019}, and Taiwanese Credit~\cite{YEH20092473}. These have two classes, with one class representing the most favorable outcome, and the other representing the least desirable outcome for which we aim to generate counterfactuals. For simplicity, we normalize the features to lie between $[0,1]$. 
\begin{table*}[t]
\caption{Experimental results.}
\label{exphel}
\centering
\begin{small}
\begin{tabular}{clcccccccc} 
\toprule
&& \multicolumn{4}{c}{\textit{$l_1$ based}} & \multicolumn{4}{c}{\textit{$l_2$ based}}  \\ 
\cmidrule(lr){3-6} \cmidrule(lr){7-10}
& Method & COST & LOF  & WI VAL. & LO VAL. & COST & LOF  & WI VAL. & LO VAL.   \\ 
\midrule
\multirow{5}{*}{\rotatebox[origin=c]{90}{{HELOC}}}
& min Cost & 0.40&0.49 & 38.8\% & 35.2\% & 0.11 & 0.75 & 13.5\%  & 13.5\%  \\
& min Cost+T-Rex:I (Ours) & 1.02& 0.38& 98.2\%  & 98.1\% & 0.29 & 0.68 & 98.5\%  & 98.2\%  \\
& min Cost+SNS&  1.20 & 0.30& 98.0\%  & 97.8\% & 0.31 & 0.64 & 97.9\%  & 97.0\%  \\ 
& ROAR &  1.69&0.41 & 92.6\% & 91.2\% & 1.91 & 0.43& 86.3 \% &  84.8\% \\ 
\cmidrule(lr){2-10}
& NN & 1.91 &0.80& 51.1\%  & 50.3\% & 0.56 & 0.80 & 51.1\%  & 50.3\%  \\
& T-Rex:NN (Ours) &  2.50 &0.92& 84.0\%  & 84.0\% & 0.77 & 0.92 & 84.0\%  & 84.0\%  \\
\midrule
\multirow{5}{*}{\rotatebox[origin=c]{90}{{GERMAN}}}
& min Cost & 1.42 &0.77& 58.8\%  & 56.7\% & 0.48 & 0.81 & 26.6\%  & 26.6\%  \\
& min Cost+T-Rex:I (Ours) & 4.81 &0.72& 98.0\%  & 96.5\% & 1.20 & 0.75 & 99.2\%  & 98.7\%  \\
& min Cost+SNS & 5.71 & 0.67& 97.5\%  & 98.1\% & 1.44 & 0.68 & 99.9\%  & 98.9\%  \\ 
& ROAR & 7.63 &0.54& 96.3\% & 92.3\% &6.81  & 0.58 &87.8\%   &85.2\%   \\ 
\cmidrule(lr){2-10}
& NN & 7.05 &1.00& 95.3\%  & 95.4\% & 2.50 & 1.00 & 95.3\%  & 95.3\%  \\
& T-Rex:NN (Ours) & 10.13 &1.00& 100\%   & 100\% & 3.04 & 1.00 & 100\%   & 100\%  \\
\midrule
\multirow{5}{*}{\rotatebox[origin=c]{90}{{CTG}}}
& min Cost & 0.21 &0.94& 74.6\%  & 70.2\% & 0.08 & 1.00 & 19.7\%  & 14.1\%  \\
& min Cost+T-Rex:I (Ours) & 1.11 &0.83& 100\%   & 98.8\% & 0.42 & 0.94 & 100\%   & 99.7\%  \\
& min Cost+SNS & 3.34 &-1.00& 100\%   & 98.2\% & 1.07 & -1.00 & 100\%   & 99.3\%  \\ 
& ROAR & 3.68 &0.64& 98.7\% & 96.4\% & 1.35 & 0.59& 98.9\% & 97.2\%   \\ 
\cmidrule(lr){2-10}
& NN & 0.39 &1.00& 70.5\%  & 67.5\% & 0.15 & 1.00 & 70.5\%  & 67.5\%  \\
& T-Rex:NN (Ours) & 2.22 &-0.33& 100\%   & 100\% & 1.00 & -0.33 & 100\%   & 100\%  \\
\bottomrule
\end{tabular}
\end{small}
\end{table*}

\textbf{Performance Metrics:}
Our metrics of interest are:
\begin{itemize}[leftmargin=*, topsep=0pt, itemsep=0pt]
    \item {Cost:} Average $l_1$ or $l_2$ distance between counterfactuals $x'$ and original points $x$. 

    \item {Validity (\%):} Percentage of counterfactuals that remain valid under the new model $M$.

    \item {LOF: } Predicts $-1$ for anomalous points, and $+1$ for inliers. A high average LOF essentially suggests the points lie on the data manifold and hence more realistic, i.e., \emph{higher is better} (see Definition ~\ref{defn:lof}). We use an existing implementation from ~\citet{scikit-lof} to compute the LOF.
\end{itemize}

\textbf{Methodology:} We begin by training a baseline neural network model and aim to find counterfactuals for data points with true negative predictions. To test the robustness of these counterfactual examples, we then train $50$ new models ($M$) and evaluate the $validity$ of the counterfactuals under different model change scenarios, which include: (i) Weight Initialization (WI): Retraining new models using the same hyperparameters but with different weight initialization by using different random seeds for each new model; and (ii) Leave Out (LO): Retraining new models by randomly removing a small portion ($1\%$) of the training data each time (with replacement) as well as different weight initialization.

\textbf{Hyperparameter selection:}
Our theoretical findings indicate that higher $k$ improves robustness, but comes at the cost of increased computational cost. We determine $k=1000$ was sufficient.  The value of $\sigma^2$ was determined by analyzing the variance of the features. In the dataset with the features between $[0,1]$, we found that a value of $\sigma^2=0.01$ produced good results.
The threshold $\tau$ is a critical aspect of our method and can be adjusted based on the desired effective validity. A higher $\tau$ value improves validity at the expense of $l_1$ or $\l_2$ cost. See Appendix \ref{apxEX} for more details.

\textbf{Baseline:} We compare our approaches with established baselines. First, we find the min Cost ($l_1$ and $l_2$) counterfactual \cite{wachter2017counterfactual} and use it as our base method for generating counterfactuals. We then compare T-Rex:I to the Stable Neighbor Search (SNS)~\cite{black2021consistent} and Robust Algorithmic
Recourse (ROAR)~\cite{upadhyay2021towards}. We evaluate the performance of our Robust Nearest Neighbor (Algorithm \ref{alg:LOF}:T-Rex:NN) against the Nearest Neighbor (NN) counterfactuals (closest data-support robust counterfactual in Definition~\ref{defn:DSRC}). We choose a value of $\tau$ to get high validity and compare cost and LOF with baselines.

\textbf{Results:} Results for HELOC, German Credit, and CTG datasets are summarized in Table \ref{exphel}. 
Observe that the min Cost counterfactual is not robust to variations in the training data or weight initialization as expected. ROAR generates counterfactuals with high validity, albeit at the expense of a higher cost. Our proposed method, T-Rex:I, significantly improves the validity of the counterfactuals compared to the minimum cost. The T-Rex:I algorithm achieves comparable validity results to the SNS method for both types of model changes, and often accomplishes this with lower costs and higher LOF. This can be observed across all three datasets for both $l_1$ and $l_2$ cost metrics. The T-Rex:NN algorithm also significantly improves the validity of the counterfactuals compared to the traditional Nearest Neighbor (NN) method and maintains a high LOF. It comes at a price of increased cost, but the counterfactuals are guaranteed to be realistic since they are data supported. Refer to Appendix \ref{apxEX} for additional results for Adult and Taiwanese credit datasets.

\textbf{Ablation:} To evaluate the efficacy of our proposed stability measure, we conduct an ablation study on the German credit dataset. We first evaluate a robustness measure that solely relies on the model's prediction of the counterfactual, denoted as $r(x',m)=m(x')$. We then examine a measure that only incorporates the mean, the average predictions for $k$ points sampled from the distribution $N(x',\sigma^2I_d)$, denoted as $r_{k,\sigma^2}(x',m)=\frac{1}{k}\sum_{x'_i \in N_{x',k}} m(x'_i)$. We compare these with our proposed robustness measure $\hat{R}_{k,\sigma^2}(x',m)$, which also takes into account the variability around the counterfactual. The results of the ablation study, for various $\tau$ thresholds, are summarized in Table ~\ref{tbl:ablation} in Appendix~\ref{apxEX}.

\section{Discussion}
We introduce an abstraction called naturally-occurring model change and propose a measure, Stability, to quantify the robustness of counterfactuals with probabilistic guarantees. We show that counterfactuals with high Stability will remain valid after potential model changes with high probability. We investigate various techniques for incorporating stability in generating robust counterfactuals and introduce the T-Rex:I and T-Rex:NN algorithms. We also make a novel conceptual connection with the body of work on model multiplicity, further emphasizing on the models that are more likely to occur.

\textbf{Limitations and Broader Impact:} The naturally-occurring model changes rest on assumptions that may not apply to all models or datasets. Our relaxed stability, although practically implementable, lacks the same theoretical guarantees as the initial stability measure. Estimating the Lipschitz constant around a counterfactual can be computationally demanding, particularly when leveraging gradient descent to optimize stability. Though generating robust counterfactuals is a key step towards trustworthy AI, it can fall short of other important factors such as fairness~\cite{sharma2019certifai,gupta2019equalizing,ley2022global,raman2023bayesian,ehyaei2023robustness}. Future research could explore links between robustness and fairness, improving the estimation of stability, or integrating Stability into training-time-based approaches for generating robust counterfactuals.

\paragraph{Disclaimer}
This paper was prepared for informational purposes in part by
the Artificial Intelligence Research group of JPMorgan Chase \& Co. and its affiliates (``JP Morgan''),
and is not a product of the Research Department of JP Morgan.
JP Morgan makes no representation and warranty whatsoever and disclaims all liability,
for the completeness, accuracy, or reliability of the information contained herein.
This document is not intended as investment research or investment advice, or a recommendation,
offer or solicitation for the purchase or sale of any security, financial instrument, financial product or service,
or to be used in any way for evaluating the merits of participating in any transaction,
and shall not constitute a solicitation under any jurisdiction or to any person,
if such solicitation under such jurisdiction or to such person would be unlawful.

\newpage
\bibliography{example_paper}

\begin{thebibliography}{38}
\providecommand{\natexlab}[1]{#1}
\providecommand{\url}[1]{\texttt{#1}}
\expandafter\ifx\csname urlstyle\endcsname\relax
  \providecommand{\doi}[1]{doi: #1}\else
  \providecommand{\doi}{doi: \begingroup \urlstyle{rm}\Url}\fi

\bibitem[Albini et~al.(2022)Albini, Long, Dervovic, and
  Magazzeni]{albini2021counterfactual}
Albini, E., Long, J., Dervovic, D., and Magazzeni, D.
\newblock Counterfactual shapley additive explanations.
\newblock \emph{ACM Conference on Fairness, Accountability, and Transparency},
  2022.

\bibitem[Alvarez-Melis \& Jaakkola(2018)Alvarez-Melis and
  Jaakkola]{alvarez2018robustness}
Alvarez-Melis, D. and Jaakkola, T.~S.
\newblock On the robustness of interpretability methods.
\newblock \emph{arXiv preprint arXiv:1806.08049}, 2018.

\bibitem[Baraniuk et~al.(2008)Baraniuk, Davenport, DeVore, and
  Wakin]{baraniuk2008simple}
Baraniuk, R., Davenport, M., DeVore, R., and Wakin, M.
\newblock A simple proof of the restricted isometry property for random
  matrices.
\newblock \emph{Constructive approximation}, 28:\penalty0 253--263, 2008.

\bibitem[Barocas et~al.(2020)Barocas, Selbst, and Raghavan]{barocas2020hidden}
Barocas, S., Selbst, A.~D., and Raghavan, M.
\newblock The hidden assumptions behind counterfactual explanations and
  principal reasons.
\newblock In \emph{Proceedings of the 2020 Conference on Fairness,
  Accountability, and Transparency}, pp.\  80--89, 2020.

\bibitem[Black et~al.(2021)Black, Wang, Fredrikson, and
  Datta]{black2021consistent}
Black, E., Wang, Z., Fredrikson, M., and Datta, A.
\newblock Consistent counterfactuals for deep models.
\newblock \emph{arXiv preprint arXiv:2110.03109}, 2021.

\bibitem[Black et~al.(2022)Black, Raghavan, and Barocas]{modelmult_black}
Black, E., Raghavan, M., and Barocas, S.
\newblock Model multiplicity: Opportunities, concerns, and solutions.
\newblock In \emph{Proceedings of the 2022 ACM Conference on Fairness,
  Accountability, and Transparency}, FAccT '22, pp.\  850–863, 2022.

\bibitem[Boucheron et~al.(2013)Boucheron, Lugosi, and
  Massart]{boucheron2013concentration}
Boucheron, S., Lugosi, G., and Massart, P.
\newblock \emph{Concentration inequalities: A nonasymptotic theory of
  independence}.
\newblock Oxford university press, 2013.

\bibitem[Breiman(2001)]{Breiman2001StatisticalMT}
Breiman, L.
\newblock Statistical modeling: The two cultures.
\newblock \emph{Quality Engineering}, 48:\penalty0 81--82, 2001.

\bibitem[Breunig et~al.(2000)Breunig, Kriegel, Ng, and Sander]{breunig2000lof}
Breunig, M.~M., Kriegel, H.-P., Ng, R.~T., and Sander, J.
\newblock Lof: identifying density-based local outliers.
\newblock In \emph{Proceedings of the 2000 ACM SIGMOD international conference
  on Management of data}, pp.\  93--104, 2000.

\bibitem[Dominguez-Olmedo et~al.(2022)Dominguez-Olmedo, Karimi, and
  Sch{\"o}lkopf]{dominguez2022adversarial}
Dominguez-Olmedo, R., Karimi, A.~H., and Sch{\"o}lkopf, B.
\newblock On the adversarial robustness of causal algorithmic recourse.
\newblock In \emph{International Conference on Machine Learning}, pp.\
  5324--5342. PMLR, 2022.

\bibitem[Dua \& Graff(2017)Dua and Graff]{Dua:2019}
Dua, D. and Graff, C.
\newblock {UCI} machine learning repository, 2017.
\newblock URL \url{http://archive.ics.uci.edu/ml}.

\bibitem[Dutta et~al.(2022)Dutta, Long, Mishra, Tilli, and
  Magazzeni]{dutta2022robust}
Dutta, S., Long, J., Mishra, S., Tilli, C., and Magazzeni, D.
\newblock Robust counterfactual explanations for tree-based ensembles.
\newblock In \emph{International Conference on Machine Learning}, pp.\
  5742--5756. PMLR, 2022.

\bibitem[Ehyaei et~al.(2023)Ehyaei, Karimi, Sch{\"o}lkopf, and
  Maghsudi]{ehyaei2023robustness}
Ehyaei, A.-R., Karimi, A.-H., Sch{\"o}lkopf, B., and Maghsudi, S.
\newblock Robustness implies fairness in casual algorithmic recourse.
\newblock \emph{arXiv preprint arXiv:2302.03465}, 2023.

\bibitem[FICO(2018)]{fico2018a}
FICO.
\newblock {FICO XML Challenge}.
\newblock
  \url{https://community.fico.com/s/explainable-machine-learning-challenge},
  2018.

\bibitem[Gupta et~al.(2019)Gupta, Nokhiz, Roy, and
  Venkatasubramanian]{gupta2019equalizing}
Gupta, V., Nokhiz, P., Roy, C.~D., and Venkatasubramanian, S.
\newblock Equalizing recourse across groups.
\newblock \emph{arXiv preprint arXiv:1909.03166}, 2019.

\bibitem[Hancox{-}Li(2020)]{Hancox-Li_fat_2020}
Hancox{-}Li, L.
\newblock {Robustness in Machine Learning Explanations: Does It Matter?}
\newblock In \emph{Proceedings of the 3rd ACM Conference on Fairness,
  Accountability, and Transparency (FAT*)}, pp.\  640--647. Barcelona, Spain,
  January 27--30 2020.

\bibitem[Hsu \& Calmon(2022)Hsu and Calmon]{hsu2022rashomon}
Hsu, H. and Calmon, F.
\newblock Rashomon capacity: A metric for predictive multiplicity in
  classification.
\newblock In \emph{Advances in Neural Information Processing Systems},
  volume~35, pp.\  28988--29000. Curran Associates, Inc., 2022.

\bibitem[Jiang et~al.(2022)Jiang, Leofante, Rago, and
  Toni]{jiang2022formalising}
Jiang, J., Leofante, F., Rago, A., and Toni, F.
\newblock Formalising the robustness of counterfactual explanations for neural
  networks.
\newblock \emph{arXiv preprint arXiv:2208.14878}, 2022.

\bibitem[Kanamori et~al.(2020)Kanamori, Takagi, Kobayashi, and
  Arimura]{kanamori2020dace}
Kanamori, K., Takagi, T., Kobayashi, K., and Arimura, H.
\newblock Dace: Distribution-aware counterfactual explanation by mixed-integer
  linear optimization.
\newblock In \emph{IJCAI}, pp.\  2855--2862, 2020.

\bibitem[Karimi et~al.(2020)Karimi, Barthe, Sch{\"{o}}lkopf, and
  Valera]{Karimi_arXiv_2020}
Karimi, A., Barthe, G., Sch{\"{o}}lkopf, B., and Valera, I.
\newblock A survey of algorithmic recourse: definitions, formulations,
  solutions, and prospects.
\newblock \emph{CoRR}, abs/2010.04050, 2020.

\bibitem[Laugel et~al.(2019)Laugel, Lesot, Marsala, and
  Detyniecki]{laugel2019issues}
Laugel, T., Lesot, M.-J., Marsala, C., and Detyniecki, M.
\newblock Issues with post-hoc counterfactual explanations: a discussion.
\newblock \emph{arXiv preprint arXiv:1906.04774}, 2019.

\bibitem[Ley et~al.(2022)Ley, Mishra, and Magazzeni]{ley2022global}
Ley, D., Mishra, S., and Magazzeni, D.
\newblock Global counterfactual explanations: Investigations, implementations
  and improvements, 2022.

\bibitem[Maragno et~al.(2023)Maragno, Kurtz, R{\"o}ber, Goedhart, Birbil, and
  Hertog]{maragno2023finding}
Maragno, D., Kurtz, J., R{\"o}ber, T.~E., Goedhart, R., Birbil, {\c{S}}.~I.,
  and Hertog, D.~d.
\newblock Finding regions of counterfactual explanations via robust
  optimization.
\newblock \emph{arXiv preprint arXiv:2301.11113}, 2023.

\bibitem[Marx et~al.(2020)Marx, Calmon, and Ustun]{marx2020predictive}
Marx, C., Calmon, F., and Ustun, B.
\newblock Predictive multiplicity in classification.
\newblock In \emph{International Conference on Machine Learning}, pp.\
  6765--6774. PMLR, 2020.

\bibitem[Mishra et~al.(2021)Mishra, Dutta, Long, and
  Magazzeni]{Mishra_arXiv_2021}
Mishra, S., Dutta, S., Long, J., and Magazzeni, D.
\newblock {A Survey on the Robustness of Feature Importance and Counterfactual
  Explanations}.
\newblock \emph{arXiv e-prints}, arXiv:2111.00358, 2021.

\bibitem[Pawelczyk et~al.(2020{\natexlab{a}})Pawelczyk, Broelemann, and
  Kasneci]{pawelczyk2020counterfactual}
Pawelczyk, M., Broelemann, K., and Kasneci, G.
\newblock On counterfactual explanations under predictive multiplicity.
\newblock In \emph{Conference on Uncertainty in Artificial Intelligence}, pp.\
  809--818. PMLR, 2020{\natexlab{a}}.

\bibitem[Pawelczyk et~al.(2020{\natexlab{b}})Pawelczyk, Broelemann, and
  Kasneci]{pawelczyk2020learning}
Pawelczyk, M., Broelemann, K., and Kasneci, G.
\newblock Learning model-agnostic counterfactual explanations for tabular data.
\newblock In \emph{Proceedings of The Web Conference 2020}, pp.\  3126--3132,
  2020{\natexlab{b}}.

\bibitem[Pawelczyk et~al.(2022)Pawelczyk, Datta, van-den Heuvel, Kasneci, and
  Lakkaraju]{pawelczyk2022probabilistically}
Pawelczyk, M., Datta, T., van-den Heuvel, J., Kasneci, G., and Lakkaraju, H.
\newblock Probabilistically robust recourse: Navigating the trade-offs between
  costs and robustness in algorithmic recourse.
\newblock \emph{arXiv preprint arXiv:2203.06768}, 2022.

\bibitem[Poyiadzi et~al.(2020)Poyiadzi, Sokol, Santos-Rodriguez, De~Bie, and
  Flach]{poyiadzi2020face}
Poyiadzi, R., Sokol, K., Santos-Rodriguez, R., De~Bie, T., and Flach, P.
\newblock Face: Feasible and actionable counterfactual explanations.
\newblock In \emph{Proceedings of the AAAI/ACM Conference on AI, Ethics, and
  Society}, pp.\  344--350, 2020.

\bibitem[Raman et~al.(2023)Raman, Magazzeni, and Shah]{raman2023bayesian}
Raman, N., Magazzeni, D., and Shah, S.
\newblock Bayesian hierarchical models for counterfactual estimation.
\newblock In \emph{International Conference on Artificial Intelligence and
  Statistics}, pp.\  1115--1128. PMLR, 2023.

\bibitem[Rawal et~al.(2020)Rawal, Kamar, and Lakkaraju]{rawal2020can}
Rawal, K., Kamar, E., and Lakkaraju, H.
\newblock Can {I} still trust you?: Understanding the impact of distribution
  shifts on algorithmic recourses.
\newblock \emph{arXiv preprint arXiv:2012.11788}, 2020.

\bibitem[Scikit-Learn()]{scikit-lof}
Scikit-Learn.
\newblock {LOF Implementation}.
\newblock URL
  \url{https://scikit-learn.org/stable/modules/generated/sklearn.neighbors.LocalOutlierFactor.html}.

\bibitem[Sharma et~al.(2019)Sharma, Henderson, and Ghosh]{sharma2019certifai}
Sharma, S., Henderson, J., and Ghosh, J.
\newblock Certifai: Counterfactual explanations for robustness, transparency,
  interpretability, and fairness of artificial intelligence models.
\newblock \emph{arXiv preprint arXiv:1905.07857}, 2019.

\bibitem[Upadhyay et~al.(2021)Upadhyay, Joshi, and
  Lakkaraju]{upadhyay2021towards}
Upadhyay, S., Joshi, S., and Lakkaraju, H.
\newblock Towards robust and reliable algorithmic recourse.
\newblock \emph{Advances in Neural Information Processing Systems}, 34, 2021.

\bibitem[Verma et~al.(2020)Verma, Dickerson, and
  Hines]{verma2020counterfactual}
Verma, S., Dickerson, J., and Hines, K.
\newblock Counterfactual explanations for machine learning: A review.
\newblock \emph{arXiv preprint arXiv:2010.10596}, 2020.

\bibitem[Wachter et~al.(2017)Wachter, Mittelstadt, and
  Russell]{wachter2017counterfactual}
Wachter, S., Mittelstadt, B., and Russell, C.
\newblock Counterfactual explanations without opening the black box: Automated
  decisions and the gdpr.
\newblock \emph{Harv. JL \& Tech.}, 31:\penalty0 841, 2017.

\bibitem[Watson-Daniels et~al.(2023)Watson-Daniels, Parkes, and
  Ustun]{watson2023predictive}
Watson-Daniels, J., Parkes, D.~C., and Ustun, B.
\newblock Predictive multiplicity in probabilistic classification.
\newblock In \emph{Proceedings of the AAAI Conference on Artificial
  Intelligence}, volume~37, pp.\  10306--10314, 2023.

\bibitem[Yeh \& hui Lien(2009)Yeh and hui Lien]{YEH20092473}
Yeh, I.-C. and hui Lien, C.
\newblock The comparisons of data mining techniques for the predictive accuracy
  of probability of default of credit card clients.
\newblock \emph{Expert Systems with Applications}, 36\penalty0 (2, Part
  1):\penalty0 2473--2480, 2009.

\end{thebibliography}
\bibliographystyle{icml2023}


\appendix
\onecolumn

\section{Relevant Inequalities}

\begin{lem}[Cauchy-Schwarz Inequality]\label{caucy}
If $\mathbf{u}, \mathbf{v} \in V$, where $V$ is a vector space, then
$$
|\langle\mathbf{u}, \mathbf{v}\rangle| \leq\|\mathbf{u}\|\|\mathbf{v}\| .
$$
This inequality is an equality if and only if one of $\mathbf{u}, \mathbf{v}$ is a scalar multiple of the other.

\end{lem} 

\begin{lem}[Jensens Inequality]\label{jens}
Let $X$ be an integrable random variable. Let $g: \mathbb{R} \rightarrow \mathbb{R}$ be a convex function such that $Y=g(X)$ is also integrable. Then, the following inequality, called Jensen's inequality, holds:
$$
\mathrm{E}[g(X)] \geq g(\mathrm{E}[X]).
$$
\end{lem}

\section{Proof of Lemma~\ref{lem:natural}} \label{apx:bound}
\natocc*
\begin{proof}

\begin{align}
\E{\frac{1}{n}\sum_{i=1}^n{|m(x_i)-M(x_i)|}} & \overset{(a)}{\leq} \E{\sqrt{\sum_{i=1}^n \frac{1}{n^2}} \sqrt{\sum_{i=1}^n |m(x_i)-M(x_i)|^2 }\ \ } \\
& = \frac{1}{\sqrt{n}}\ \E{\sqrt{\sum_{i=1}^n |m(x_i)-M(x_i)|^2} \ \ } \\
& \overset{(b)}{\leq} \frac{1}{\sqrt{n}}\ \sqrt{\E{\sum_{i=1}^n |m(x_i)-M(x_i)|^2} \ \ } \\
& = \frac{1}{\sqrt{n}}\ \sqrt{\sum_{i=0}^n\nu_{x_i}} \\
& \overset{(c)}{\leq } \frac{1}{\sqrt{n}}\ \sqrt{\sum_{i=0}^n\nu} = \sqrt{\nu}.
 \end{align}
Here (a) holds from Cauchy-Schwarz Inequality (Lemma~\ref{caucy}) applied on the dot product of the two vectors $[1/n,1/n,\ldots,1/n]$ and $[|m(x_1)-M(x_1)|, |m(x_2)-M(x_2)|, \ldots, |m(x_n)-M(x_n)|]$. Next, (b) holds from Jensen's Inequality (Lemma~\ref{jens}) applied on concave function $f(u)=\sqrt u $. Finally, (c) holds because the points $x_1,x_2,\ldots,x_n$ lie on the data-manifold and hence the variance $\nu_{x_i} \leq \nu$ from Definition~\ref{defn:natural}.

\end{proof}

\section{Proof of Probabilistic Guarantee}\label{apxproof}
To prove Theorem \ref{thm:guarantee}, we begin with Lemma \ref{lem:concentration_modified}.

Assume the changed model \(M\) comes from a discrete class of random variables. A possible realization of a model is denoted by \(\tilde{m}_i\), with \(i = 1, 2, \ldots, n\). The set of all possible models is denoted by \(\mathcal{M} = \{ \tilde{m}_1, \tilde{m}_2, \ldots, \tilde{m}_n \}\). Let $M=\Tilde{m}_i$ with probability $p_i$ such that $\sum_{i=1}^n p_i =1$.
\bound*

\begin{proof}
To prove Lemma~\ref{lem:concentration_modified}, notice that, 
\begin{align}
\E{Z}&=\frac{1}{k}\sum_{i=1}^k\Esub{X_i}{\Esub{M|X_i}{m(X_i)-M(X_i)}} \nonumber \\
& \overset{(a)}{=} \frac{1}{k}\sum_{i=1}^k\Esub{X_i}{m(X_i)-m(X_i)} = 0,
\end{align}
where (a) holds from the naturally occurring model change assumption (see Definition~\ref{prop1}). The remaining part of the proof leverages concentration bounds for Lipschitz functions of independent Gaussian random variables outlined in Lemma~\ref{lem:concentration_actual}. 
\gaus*

Now, let $X_{ij}$ denote the $j$-th element of $X_i \in \mathbb{R}^d$ and $x_j$ denote the $j$-th element of $x \in \mathbb{R}^d$. We define a $k \times d$ matrix $W=[W_{ij}]_{i=1,2,\ldots,k, \text{ and } j=1,2,\ldots,d}$ with $W_{ij}= X_{ij}-x_j$. Notice that, $W_{ij} \sim \mathcal{N}(0,\sigma^2)$ for $i=1,2,\ldots,k$ and $j=1,2,\ldots,d$ and, we can write $Z=f(W)$ with Lipschitz constant $(\gamma_m + \gamma)/\sqrt{k}$. 
\begin{align}
&|f(W) - f(W')| \nonumber \\
& =\left| \frac{1}{k}\sum_{i=1}^k (m(X_i)-M(X_i) - m(X_i') + M(X_i')) \right| \nonumber\\
& \overset{(a)}{\leq} \frac{1}{k}\sum_{i=1}^k (|m(X_i)-m(X_i')| + |M(X_i)-M(X_i')|) \nonumber\\
& \overset{(b)}{\leq} \frac{1}{k}\sum_{i=1}^k (\gamma_m + \gamma)\|X_i -X_i'\|_2 \nonumber\\
& \overset{(c)}{\leq} \frac{(\gamma_m + \gamma)\sqrt{k}}{k} \sqrt{\sum_{i=1}^k \|X_i -X_i'\|_2^2} \nonumber \\
& = \frac{(\gamma_m + \gamma)}{\sqrt{k}} \sqrt{\sum_{i=1}^k \sum_{j=1}^d |W_{ij}-W'_{ij}|^2} \nonumber \\
& = \frac{(\gamma_m + \gamma)}{\sqrt{k}} \|W-W'\|_2.
\end{align}
\end{proof}
Here, (a) holds from the triangle inequality. (b) follows directly from the definition of $\gamma$-Lipschitz (Definition \ref{def:Lipschitz}). (c) can be obtained by using the Cauchy-Schwarz Inequality (Lemma \ref{caucy}).

Now, we substitute these expressions in the Gaussian concentration bound (Lemma~\ref{lem:concentration_actual}). 

\begin{align}
\Pr(Z - \E{Z| M= \til{m}} \geq \til{\epsilon}| M = \til{m}) \leq \exp{\left(\frac{-k\til{\epsilon}^2}{2(\gamma + \gamma_m)^2\sigma^2}\right)}.
\end{align}

Since $|\E{Z|M}- \E{Z}| < \epsilon'$ and $\E{Z}=0$, we have  $ - \epsilon' < \E{Z|M= \til{m}} <  \epsilon'$  $\forall \til{m} \in \mathcal{M}$. Now observe that,

\begin{align}\label{eqn: Z_cond}
\Pr(Z \geq \epsilon' + \til{\epsilon}| M = \til{m}) &{\stackrel{\text{(a)}}{\leq}} \Pr(Z \geq \E{Z| M= \til{m}} + \til{\epsilon}| M = \til{m}) \\
& \leq \exp{\left(\frac{-k\til{\epsilon}^2}{2(\gamma + \gamma_m)^2\sigma^2}\right)}.
\end{align}

Here, (a) holds since $\E{Z|M= \til{m}} <  \epsilon'$. The event on the left is a subset of that on the right. Therefore, the probability of the event \{$Z \geq \epsilon' + \til{\epsilon}$\}   occurring cannot be more than the probability of the event \{$Z \geq \E{Z| M= \til{m}} + \til{\epsilon}$\}  occurring.
\begin{align}
    \Pr(Z \geq \epsilon' + \Tilde{\epsilon})
    &{\stackrel{\text{(b)}}{=}} \sum_{i=1}^n \Pr(Z \geq  \epsilon'+\Tilde{\epsilon} | M = \Tilde{m}_i) \Pr(M = \til{m_i})\\
    & {\stackrel{\text{(c)}}{\leq}} \exp \left(\frac{-k \til{\epsilon}^2}{2(\gamma_m+\gamma)^2 \sigma^2}\right) \sum_{i=1}^n \Pr(M = \til{m_i})\\
    & = \exp \left(\frac{-k \til{\epsilon}^2}{2(\gamma_m+\gamma)^2 \sigma^2}\right)\\
    &{\stackrel{\text{(d)}}{\leq}} \exp \left(\frac{-k (\Tilde{\epsilon}+ \epsilon')^2}{8(\gamma_m+\gamma)^2 \sigma^2}\right)
\end{align}
Here, (b) holds from the law of total probability. Next, (c) follows from bound in \eqref{eqn: Z_cond}. Finally, (d) holds from the inequality $4 \til{\epsilon}^2 > (\til{\epsilon} + \epsilon')^2$ which holds for $\til{\epsilon} > \epsilon' > 0 $. Setting $\epsilon = \til{\epsilon} + \epsilon'$ completes the proof of Lemma \ref{lem:concentration_modified}

\subsection{Proof of Theorem \ref{thm:guarantee}}\label{app:prof_guar}
\Guarantee*
\begin{proof}
The Lipschitz property of $M(\cdot)$ around $x$ is given by,
$$|M(x)-M(x')|\leq \gamma\|x-x'\| \text{ for all } x,x' \in \mathbb{R}^d$$
Therefore,
\begin{equation}
M(x)  \geq M(X_i)-\gamma\|x-X_i\|
\label{eq:lip_ineq}
\end{equation}
\begin{equation}
M(x) \overset{(a)}{\geq} \frac{1}{k}\sum_{i=1}^k (M(X_i)-\gamma\|x-X_i\|)  
\label{eqn:apxThm}
\end{equation}
where (a) holds from taking the average of the inequality \eqref{eq:lip_ineq} over all $i$ from $1$ to $k$.\\
From Lemma \ref{lem:concentration_modified}, for $X_1,X_2,\ldots,X_k \sim \mathcal{N}(x,\sigma^2I_d)$, 
\begin{equation}
\frac{1}{k}\sum_{i=1}^k M(X_i) \geq \frac{1}{k}\sum_{i=1}^k m(X_i) - \epsilon   
\label{eqn:conc}
\end{equation}
with probability at least $1-\exp{\left(-\frac{k\epsilon^2}{8(\gamma + \gamma_m)^2\sigma^2}\right)}.$

Hence, plugging \eqref{eqn:conc} into \eqref{eqn:apxThm}, we have:
$$\Pr \big(M(x)\geq \frac{1}{k}\sum_{i=1}^k (m(X_i)-\gamma\|x-X_i\|)-\epsilon \big) \geq 1-\exp{\left(\frac{-k\epsilon^2}{8(\gamma+\gamma_m)^2\sigma^2}\right)}.$$

Recall from Definition \ref{def:stability}, stability $R_{k,\sigma^2}(x,m)=\frac{1}{k}\sum_{i=1}^k \left( m(X_i)  - \gamma \|x-X_i\|\right)$. Hence, we have:
$$\Pr(M(x)\geq R_{k,\sigma^2}(x,m) - \epsilon)
\geq 1 - \exp{\left(\frac{-k\epsilon^2}{8(\gamma+\gamma_m)^2\sigma^2}\right)}.$$
\end{proof}

\section{Appendix to Experiments in Section~\ref{sec:experiments}}\label{apxEX}
\subsection{Datasets}
\textbf{HELOC.} The FICO HELOC ~\cite{fico2018a} dataset contains anonymized information about home equity line of credit applications made by homeowners in the US, with a binary response indicating whether or not the applicant has ever been more than 90 days delinquent for a payment. It can be used to train a machine learning model to predict whether the homeowner qualifies for a line of credit or not. The dataset consists of 10459 rows and 40 features, which we have normalized to be between zero and one.

\textbf{German Credit.} The German Credit Dataset~\cite{Dua:2019} comprises 1000 entries, each representing an individual who has taken credit from a bank. These entries are characterized by 20 categorical features, which are used to classify each person as a good or bad credit risk. To prepare the dataset for analysis, we one-hot encoded the data and normalized it such that all features fall between 0 and 1. Additionally, we partitioned the dataset into a training set and a test set, with a 70:30 ratio respectively.

\textbf{CTG.} The CTG dataset~\cite{Dua:2019} consists of 2126 fetal cardiotocograms, which have been evaluated and categorized by experienced obstetricians into three categories: healthy, suspect, and pathological. We process this dataset based on \citet{black2021consistent}. The problem was transformed into a binary classification task, where healthy fetuses are distinguished from the other two categories. We divided the dataset into a training set of 1,700 instances and a validation set of 425 instances. Each instance is described by 21 features, which we normalized to have values between zero and one.

\textbf{Adult.} The Adult dataset ~\cite{Dua:2019} is a publicly available dataset in the UCI repository based on $1994$ U.S. census data. This dataset is a classification task to successfully predict whether an individual earns more or less than $50,000$ per year based on features such as occupation, marital status, and education, etc. The dataset consists of approximately 48,842 instances, split into a training set of 32,561 instances and a test set of 16,281 instances. 

\textbf{Taiwanese Credit.} The Taiwanese dataset~\cite{YEH20092473} consists of 30,000 instances with 24 features that include individuals' financial data, with a binary response indicating their creditworthiness. We use one-hot encoding on the data and normalize it to be between zero and one. This dataset is processed based on ~\citet{black2021consistent}. We partition the data into a training set of 22,500 and a test set of 7,500.

\subsection{Model Architecture } We first trained a base neural network model for which we aim to find counterfactuals for instances with false negative predictions. The architecture of our base model consisted of two hidden layers, each containing 128 hidden units. We employed the rectified linear unit (ReLU) activation function and Adam optimizer. The model was trained for 50 epochs using a batch size of 32. We employed this model architecture and training setup on all datasets since it yielded a satisfactory level of accuracy on all of them. To evaluate the robustness of these counterfactual examples, we then trained 50 new models and assessed the validity of the counterfactuals under various model change scenarios. All $50$ models had the same architecture and training setup as the base model except for some slight changes which include: Weight Initialization (WI): Retraining new models with the same hyperparameters but different weight initialization by using different random seeds for each model. Leave Out (LO): Retraining new models by randomly removing a small portion ($1\%$) of the training data each time and using different weight initialization.

\subsection{Implementation Details}
The stability measure has the number of samples $k$ and the variance $\sigma^2$ as hyperparameters. Our theoretical findings suggest that a higher value of $k$ improves the robustness of the counterfactuals but at the expense of increased cost and computational complexity. In our experiments, we found that a value of $k=1000$ was sufficient. The value of $\sigma^2$ was determined by analyzing the variance of the features in the dataset, and we found that a value of $\sigma^2=0.01$ produced good results for the features that lie between $[0,1]$. These hyperparameters were kept constant across all our experiments and datasets. The baseline technique for generating counterfactuals used in T-rex was the min Cost counterfactual \cite{wachter2017counterfactual}. For other algorithm parameters, a step size of $0.01$ was fixed for all datasets and experiments. The maximum number of iterations (max$\_$steps) varied depending on the dataset, with 50 for HELOC, 200 for German Credit, and 100 for CTG, Adult, and Taiwanese. An appropriate value of $\tau$ balances the trade-off between validity and cost. We choose a value of $\tau$ to guarantee high validity and compare cost and LOF with the baselines. Another method of choosing $\tau$ is to use the histogram of $R(\cdot)$
 on the training dataset (e.g., see Fig. \ref{fig:hist} for HELOC dataset). To implement ROAR and SNS, we adhered to techniques and algorithm parameters discussed in the original works \cite{upadhyay2021towards,black2021consistent}. In NN and T-Rex:NN implementation, a crucial consideration is determining the appropriate number of neighbors $K$, to search for a robust counterfactual. For larger datasets such as HELOC, Adult, and Taiwanese datasets $K=1000$, while for the German credit and CTG datasets $K=100$. Note that if $K$ is too small, a counterfactual robustness test might not be satisfied, and hence T-Rex:NN returns no robust counterfactuals.

\begin{figure}[t]
\centering
    \begin{subfigure}[b]{0.3\textwidth}
        \includegraphics[width=0.9\textwidth]{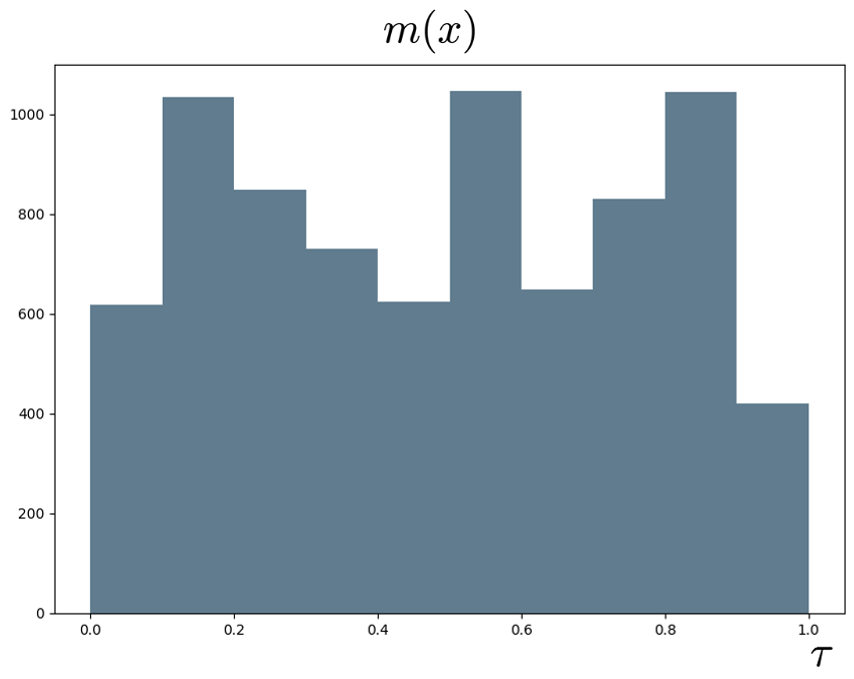}
    \end{subfigure}
    \hfill
    \begin{subfigure}[b]{0.3\textwidth}
        \includegraphics[width=0.9\textwidth]{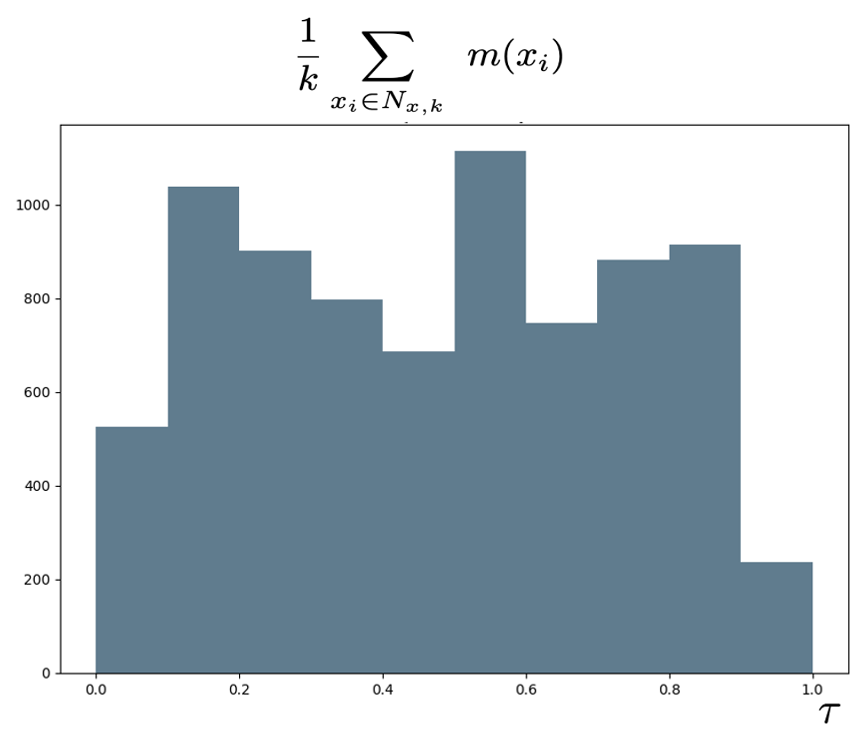}
    \end{subfigure}
    \hfill
    \begin{subfigure}[b]{0.3\textwidth}
        \includegraphics[width=0.9\textwidth]{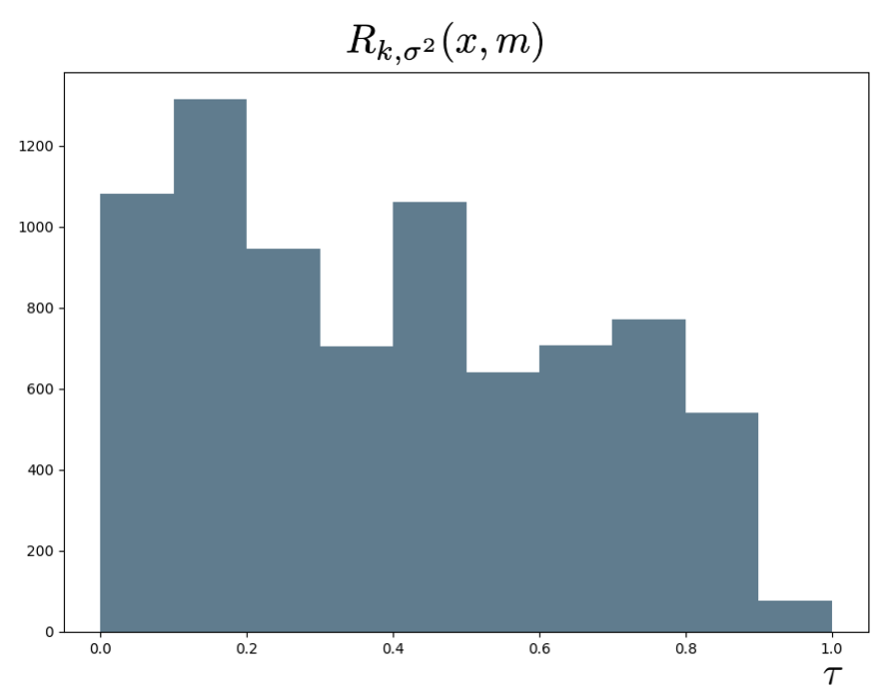}
    \end{subfigure}
    \caption{ Histograms on the HELOC dataset to visualize the proposed stability measure.}
    \label{fig:hist}
\end{figure}
\subsection{Additional Experimental Results}

\begin{table}[H]
\caption{Experimental results for HELOC dataset with standard deviations}
\centering
\begin{small}
\begin{tabular}{llccccccc} 
\toprule
\multirow{2}{*}{{HELOC}}    & \multicolumn{4}{c}{\textit{$l_1$ based}} & \multicolumn{4}{c}{\textit{$l_2$ based}}  \\ 
\cmidrule(lr){2-5} \cmidrule(lr){6-9}
   & COST & LOF  & WI VAL. & LO VAL. & COST & LOF  & WI VAL. & LO VAL.   \\ 
\midrule
min Cost & $0.40 \pms 0.48$ & $0.49 \pms 0.93$ & $38.8 \pms 2.5\% $& $35.2 \pms 2.5\%$  & $0.11 \pms 0.08$ & $0.75 \pms 0.61$ & $13.5\pms 3.4\%$  & $13.5\pms 3.5\%$\\
+T-Rex:I & $1.02 \pms 0.41$ & $0.38 \pms 0.95$ & $98.2\pms 1.2\%$  & $98.1\pms 1.2\%$ & $0.29\pms 0.07$ & $0.68\pms 0.64$ & $98.5 \pms 1.9\%$  & $98.2 \pms 2.0\%$  \\
+SNS & $1.20\pms0.46$ & $0.30\pms0.46$ & $98.0\pms 1.2\% $ & $97.8\pms 1.1\%$  & $0.31\pms0.08$ & $0.64\pms 0.78$ & $97.9\pms 0.8\%$  & $97.0\pms 0.8\%$ \\
ROAR & $1.69 \pms 1.59 $ & $0.41 \pms 0.73$ & $92.6 \pms 3.9\%$ & $91.2\pms 4.3\%$ & $1.91\pms 2.22$ & $0.43\pms 0.83$& $86.3\pms 0.2\%$ &  $84.8\pms 0.3\%$ \\ 
\cmidrule(lr){1-5} \cmidrule(lr){6-9}
NN & $1.91\pms 2.27$ & $0.80\pms2.27$ & $51.1\pms 12.6\%$  & $50.3\pms 12.4\%$ & $0.56\pms 0.59$ & $0.80\pms 2.27$ & $51.1\pms 12.6\%$  & $50.3\pms 12.4\%$   \\
T-Rex:NN & $2.50\pms 1.83$ & $0.92\pms0.56$ & $84.0\pms 0.8\%$  & $84.0\pms 0.8\%$ & $0.77\pms 0.61$ & $0.92\pms0.56$ & $84.0\pms 0.8\%$  & $84.0\pms 0.8\% $\\
\bottomrule
\end{tabular}
\end{small}
\end{table}

\begin{table}[H]
\caption{Experimental results for German Credit dataset with standard deviations.}
\label{expgerm}
\centering
\begin{small}
\begin{tabular}{llccccccc} 
\toprule
\multirow{2}{*}{{GERMAN}}    & \multicolumn{4}{c}{\textit{$l_1$ based}} & \multicolumn{4}{c}{\textit{$l_2$ based}}  \\ 
\cmidrule(lr){2-5} \cmidrule(lr){6-9}
   & COST & LOF  & WI VAL. & LO VAL. & COST & LOF  & WI VAL. & LO VAL.   \\ 
\midrule
min Cost & $1.42 \pms 4.16$ & $0.77 \pms 0.65$ & $58.8\pms 7.9\% $  & $56.7 \pms 7.3\% $ & $0.48 \pms 1.38$ & $0.81 \pms 0.71$ & $26.6\pms 15\% $  & $26.6\pms 14\%$  \\
+T-Rex:I     & $4.81 \pms 3.93$ & $0.72 \pms0.73$ & $98.0 \pms 4.6 \%$  & $96.5\pms 3.8\%$ & $1.20 \pms 1.37$ & $0.75 \pms 0.71$ & $99.2 \pms 0.9\%$  & $98.7\pms 0.9\%$  \\
+SNS     & $5.71 \pms4.10$ & $0.67 \pms0.39$ & $97.5 \pms 0.3\%$  & $98.1\pms 0.2\%$ & $1.44 \pms 1.38$ & $0.68 \pms 0.73$ & $99.9\pms 0.0 \%$  & $98.9\pms 0.1\%$  \\ 
ROAR     & $7.63 \pms 4.00$ & $0.54 \pms 0.67$ & $96.3 \pms 0.2\%$ & $92.3 \pms 0.3\%$ & $6.81 \pms 1.00$  & $0.58 \pms0.98$ & $87.8 \pms 1.4\%$ & $85.2\pms 1.6\%$   \\ 
\cmidrule(lr){1-5} \cmidrule(lr){6-9}
NN       & $7.05 \pms 3.96$ & $1.00 \pms 0.00$ & $95.3 \pms 0.9\%$  & $95.4 \pms 1.0\%$ & $2.50 \pms 1.20$ & $1.00 \pms 0.00$ & $95.3\pms 0.9$  & $95.3\pms 1.0 \%$  \\
T-Rex:NN    & $10.13 \pms 4.10 $ & $1.00 \pms 0.00$ & $100 \pms 0.0 \%$   & $100 \pms 0.0 \%$ & $3.04 \pms 1.45$ & $1.00 \pms 0.00$ & $100 \pms 0.0 \%$   & $100 \pms 0.0 \%$  \\
\bottomrule
\end{tabular}
\end{small}
\end{table}

\begin{table}[H]
\caption{Experimental results for CTG dataset with standard deviations.}
\label{expctg}
\centering
\begin{small}
\begin{tabular}{llccccccc} 
\toprule
\multirow{2}{*}{{CTG}}    & \multicolumn{4}{c}{\textit{$l_1$ based}} & \multicolumn{4}{c}{\textit{$l_2$ based}}  \\ 
\cmidrule(lr){2-5} \cmidrule(lr){6-9}
   & COST & LOF  & WI VAL. & LO VAL. & COST & LOF  & WI VAL. & LO VAL.   \\ 
\midrule
min Cost & $0.21 \pms 0.18$ & $0.94 \pms 0.91$ & $74.6 \pms 0.1\% $  & $70.2 \pms 0.2\% $ & $0.08 \pms 0.04$ & $1.00 \pms 0.00 $ & $19.7\pms 30\%$  & $14.1\pms 31\% $  \\
+T-Rex:I     & $1.11 \pms 0.11$ & $0.83 \pms 0.91$ & $100 \pms 0.0\%$   & $98.8 \pms 0.1\%$ & $0.42 \pms 0.04$ & $0.94 \pms 0.63$ & $100\pms 0.0\%$   & $99.7\pms 0.1\%$  \\
+SNS     & $3.34 \pms 0.18$ & $-1.00 \pms 0.0$ & $100 \pms 0.0\% $   & $98.2 \pms 0.1\%$ & $1.07 \pms 0.04$ & $-1.00 \pms 0.0$ & $100 \pms 0.0\% $   & $99.3 \pms 0.1\%$  \\ 
ROAR     & $3.68 \pms 3.48$ & $0.64 \pms 0.78$ & $98.7 \pms 0.5\%$ & $96.4 \pms 0.3\%$ & $1.35 \pms 2.01$ & $0.59 \pms 0.90$ & $98.8 \pms 0.0\%$ & $97.2 \pms 0.0 \%$   \\ 
\cmidrule(lr){1-5} \cmidrule(lr){6-9}
NN       & $0.39 \pms 0.23$ & $1.00 \pms 0.00$ & $70.5\pms 0.2\% $  & $67.5 \pms 0.1\%$ & $0.15 \pms 0.07$ & $1.00 \pms 0.00$ & $70.5 \pms 0.2 \% $  & $67.5\pms 0.1 \%$  \\
T-Rex:NN      & $2.22 \pms 0.12 $ & $-0.33 \pms 0.67$ & $100 \pms 0.0\%$   & $100\pms 0.0\% $ & $1.00 \pms 0.80$ & $-0.33 \pms 0.67$ & $100 \pms 0.0\%$   & $100\pms 0.0 \%$  \\
\bottomrule
\end{tabular}
\end{small}
\end{table}

\begin{table}[H]
\caption{Experimental results for Taiwanese Credit dataset with standard deviations.}
\label{exptaiwan}
\centering
\begin{small}
\begin{tabular}{llccccccc} 
\toprule
\multirow{2}{*}{{\small TAIWANESE}}    & \multicolumn{4}{c}{\textit{$l_1$ based}} & \multicolumn{4}{c}{\textit{$l_2$ based}}  \\ 
\cmidrule(lr){2-5} \cmidrule(lr){6-9}
   & COST & LOF  & WI VAL. & LO VAL. & COST & LOF  & WI VAL. & LO VAL.   \\ 
\midrule
min Cost & $3.95\pms3.42$ & $-0.37\pms0.92$ & $38.4\pms 6.0\%$& $38.4\pms5.9\%$ & $2.84\pms1.16$ & $-0.68\pms0.72$ & $21.1\pms2.1\%$ & $20.0\pms2.2\%$  \\
+T-Rex:I    & $6.34 \pms 3.26$ & $0.48\pms0.67$ & $96.9\pms0.7\%$ & $96.2\pms0.7\%$ & $3.06\pms1.11$ & $0.40\pms0.31$ & $ 96.8\pms2.1\%$ & $96.4\pms1.9\%$ \\
+SNS     & $6.51\pms 3.37$ & $0.39\pms 0.39 $ & $97.2\pms1.3\%$ & $96.9\pms1.4\%$ & $3.10\pms 1.16$ & $0.43\pms0.62$ & $ 96.7\pms3.1\%$ & $96.1\pms 2.7\%$ \\
\cmidrule(lr){1-5} \cmidrule(lr){6-9}
NN       & $5.80\pms 3.82$ & $0.84\pms 0.76$ & $53.5\pms9.9\%$ & $51.6\pms8.9\%$ & $2.18\pms1.27$ & $0.84\pms 0.76$ & $53.5\pms9.9\%$ & $51.6\pms8.9\%$ \\
T-Rex:NN     & $6.89\pms4.51$ &$0.86\pms0.61$& $98.8\pms 0.9\%$ & $98.4\pms0.8\%$ & $3.54\pms1.48$ & $0.86\pms0.61$ & $98.8\pms 0.9\%$ & $98.4\pms0.8\%$ \\
\bottomrule
\end{tabular}
\end{small}
\end{table}

\begin{table}[H]
\caption{Experimental results for Adult dataset with standard deviations.}
\label{expAdult}
\centering
\begin{small}
\begin{tabular}{llccccccc} 
\toprule
\multirow{2}{*}{{ADULT}}    & \multicolumn{4}{c}{\textit{$l_1$ based}} & \multicolumn{4}{c}{\textit{$l_2$ based}}  \\ 
\cmidrule(lr){2-5} \cmidrule(lr){6-9}
   & COST & LOF  & WI VAL. & LO VAL. & COST & LOF  & WI VAL. & LO VAL.   \\ 
\midrule
min Cost & $0.12\pms0.54$ & $0.09\pms0.99$ & $80.8\pms5.0\%$  & $78.2\pms4.7\%$ & $0.18\pms0.65$ & $0.09\pms0.99$ & $50.0\pms8.0\%$  & $49.2\pms8.2\%$  \\
+T-Rex:I    & $0.51\pms0.53$ & $0.08\pms0.99$ & $98.4\pms1.6\%$   & $98.1\pms1.3\%$ & $0.24\pms0.64$ & $0.09\pms0.99$ & $98.2\pms0.0\%$   & $98.2\pms0.0\%$  \\
+SNS     & $0.92\pms0.58$ & $-0.2\pms0.21$ & $98.6\pms0.0\%$   & $97.9\pms0.1\%$ & $0.37\pms0.62$ & $-0.22\pms0.97$ & $97.9\pms0.0\%$   & $97.8\pms0.0\%$  \\ 
\cmidrule(lr){1-5} \cmidrule(lr){6-9}
NN       & $2.16\pms1.43$ & $0.91\pms0.03$ & $81.0\pms2.8\%$  & $81.0\pms2.7\%$ & $1.21\pms0.65$ & $0.91\pms0.03$ & $81.0\pms2.8\%$  & $81.0\pms2.7\%$  \\
T-Rex:NN     & $3.25\pms1.6$ & $0.85\pms0.02$ & $99.2\pms0.0\%$   & $99.0\pms0.0\%$ & $1.59\pms0.57$ & $0.85\pms0.02$ & $99.2\pms0.0\%$   & $99.0\pms0.0\%$  \\
\bottomrule
\end{tabular}
\end{small}
\end{table}

\begin{table}[H]
\centering
\caption{Ablation study on German Credit Dataset.}
\label{tbl:ablation}
\begin{small}
\begin{tabular}{lccccccc} 
\toprule
                     &  & \multicolumn{3}{c}{$l_1$ based}                                  & \multicolumn{3}{c}{$l_2$ based}                                  \\
\cmidrule(lr){3-5} \cmidrule(lr){6-8}
$\tau  $                & {Measure} & COST  & LOF   & WI VAL.(\%) & COST  & LOF   & WI VAL.(\%) \\
\midrule
\multirow{3}{*}{0.5} & $r(x,m)$  & 1.12  & 0.78  & 57.0 & 0.54  & 0.81  & 32.6 \\
                     & $r_{k,\sigma^2}(x,m)$   & 2.06  & 0.81  & 57.4 & 0.57  & 0.80  & 37.3 \\
                     & $R_{k,\sigma^2}(x,m)$      & 2.33  & 0.80  & 66.9 & 0.65  & 0.80  & 52.0 \\
\cmidrule(lr){1-2} \cmidrule(lr){3-5} \cmidrule(lr){6-8}
\multirow{3}{*}{0.6} & $r(x,m)$  & 1.48 & 0.77  & 61.7 & 0.56  & 0.80  & 38.9 \\
                     & $r_{k,\sigma^2}(x,m)$   & 2.11  & 0.77  & 62.5 & 0.58  & 0.79  & 43.0 \\
                     & $R_{k,\sigma^2}(x,m)$      & 2.40  & 0.76  & 72.9 & 0.68  & 0.81  & 61.7 \\
\cmidrule(lr){1-2} \cmidrule(lr){3-5} \cmidrule(lr){6-8}
\multirow{3}{*}{0.7} & $r(x,m)$  & 1.73  & 0.71  & 61.3 & 0.63  & 0.87  & 42.8 \\
                     & $r_{k,\sigma^2}(x,m)$   & 2.61  & 0.75  & 63.7 & 0.67  & 0.87  & 51.8 \\
                     & $R_{k,\sigma^2}(x,m)$      & 2.90  & 0.76  & 72.6 & 0.76  & 0.88  & 67.0 \\
\cmidrule(lr){1-2} \cmidrule(lr){3-5} \cmidrule(lr){6-8}
\multirow{3}{*}{0.8} & $r(x,m)$  & 1.69  & 0.76  & 74.4 & 0.68  & 0.85  & 61.0 \\
                     & $r_{k,\sigma^2}(x,m)$   & 2.50  & 0.79  & 79.0 & 0.74 & 0.85  & 73.3 \\
                     & $R_{k,\sigma^2}(x,m)$      & 2.77  & 0.77  & 86.0 & 0.82  & 0.84  & 84.2 \\
\cmidrule(lr){1-2} \cmidrule(lr){3-5} \cmidrule(lr){6-8}
\multirow{3}{*}{0.9} & $r(x,m)$   & 2.24  & 0.73  & 79.5 & 0.76  & 0.81 & 71.0 \\
                     & $r_{k,\sigma^2}(x,m)$  & 3.01  & 0.74 & 84.8 & 0.83  & 0.79  & 82.7 \\
                     & $R_{k,\sigma^2}(x,m)$      & 3.30  & 0.74 & 89.6 & 0.91  & 0.78 & 89.0 \\
\bottomrule
\end{tabular}
\end{small}
\end{table}

\end{document}